\documentclass[journal,comsoc]{IEEEtran}
\IEEEoverridecommandlockouts
\usepackage[normalem]{ulem}
\usepackage{graphicx}
\usepackage{textcomp}
\usepackage[utf8]{inputenc}
\usepackage[font=footnotesize,skip=0pt]{caption}
\usepackage{subcaption}
\usepackage{enumitem}
\usepackage[T1]{fontenc} 
\usepackage{graphicx,booktabs,multirow,bm}

\newcommand*{\HH}{_{\mathsf{H}}}

\usepackage{algorithm}
\usepackage{algorithmic}
\usepackage{array}
\usepackage{tabularx}
\usepackage{amsmath,amssymb,amsfonts,amsthm}
\usepackage{bbm}
\theoremstyle{plain}

\newtheorem{assumption}{\bf{Assumption}}
\newtheorem{theorem}{\bf{Theorem}}
\newtheorem{corollary}{\bf{Corollary}}
\newtheorem{lemma}{\bf{Lemma}}
\newtheorem{remark}{\bf{Remark}}

\usepackage{xcolor}
\usepackage{bbding}
\usepackage{hyperref}
\usepackage{url}
\usepackage{cite}
\def\BibTeX{{\rm B\kern-.05em{\sc i\kern-.025em b}\kern-.08em
    T\kern-.1667em\lower.7ex\hbox{E}\kern-.125emX}}

\allowdisplaybreaks[4]

\begin{document}

\title{Channel and Gradient-Importance Aware Device Scheduling for Over-the-Air Federated Learning}

\author{Yuchang~Sun,~\IEEEmembership{Graduate Student Member,~IEEE},
Zehong~Lin,~\IEEEmembership{Member,~IEEE}, Yuyi~Mao,~\IEEEmembership{Member,~IEEE},
Shi~Jin,~\IEEEmembership{Senior Member,~IEEE},
and Jun~Zhang,~\IEEEmembership{Fellow,~IEEE}

\thanks{            
        Part of this paper was presented at the 2023 IEEE 23rd International Conference on Communication Technology (ICCT) \cite{icct}. (\emph{Corresponding author: Zehong Lin.})
        
      	Yuchang Sun, Zehong Lin, and Jun Zhang are with the Department of Electronic and Computer Engineering, The Hong Kong University of Science and Technology, Hong Kong (e-mail: yuchang.sun@connect.ust.hk; eezhlin@ust.hk; eejzhang@ust.hk).
       
        Yuyi Mao is with the Department of Electrical and Electronic Engineering, The Hong Kong Polytechnic University, Hong Kong (e-mail: yuyi-eie.mao@polyu.edu.hk).
        
        Shi Jin is with the National Mobile Communications Research Laboratory, Southeast University, Nanjing, 210096, China (e-mail: jinshi@seu.edu.cn).
 
        }
        
}
\maketitle

\begin{abstract}

Federated learning (FL) is a popular privacy-preserving distributed training scheme, where multiple devices collaborate to train machine learning models by uploading local model updates. To improve communication efficiency, over-the-air computation (AirComp) has been applied to FL, which leverages analog modulation to harness the superposition property of radio waves such that numerous devices can upload their model updates concurrently for aggregation. However, the uplink channel noise incurs considerable model aggregation distortion, which is critically determined by the device scheduling and compromises the learned model performance. In this paper, we propose a probabilistic device scheduling framework for over-the-air FL, named \emph{PO-FL}, to mitigate the negative impact of channel noise, where each device is scheduled according to a certain probability and its model update is reweighted using this probability in aggregation. 
We prove the unbiasedness of this aggregation scheme and demonstrate the convergence of PO-FL on both convex and non-convex loss functions.
Our convergence bounds unveil that the device scheduling affects the learning performance through the \emph{communication distortion} and \emph{global update variance}. Based on the convergence analysis, we further develop a channel and gradient-importance aware algorithm to optimize the device scheduling probabilities in PO-FL. Extensive simulation results show that the proposed PO-FL framework with channel and gradient-importance awareness achieves faster convergence and produces better models than baseline methods.
\end{abstract}


\section{Introduction}\label{sec:intro}

In recent years, artificial intelligence (AI) and machine learning (ML) have experienced significant breakthroughs.
Various AI applications such as smart transportation and virtual reality (VR) generate a massive volume of data on devices at the edge of wireless networks.
These data, in turn, stimulate the development of AI as they can be utilized to train powerful ML models for numerous intelligent applications.
The conventional approach of utilizing these data is to upload them to a central server for centralized model training.
This approach, however, incurs a high communication cost and poses a severe risk of data leakage.
To address these issues, federated learning (FL) \cite{fedavg} was proposed as a privacy-preserving distributed training scheme, where the devices train their models locally based on private data and periodically send their local model updates to the central server. The server aggregates the received model updates to generate a new global model for the next round of training. As there is no need for devices to share local data, FL enables collaborative training while preserving data privacy.

Despite the benefits of FL, it suffers a significant communication bottleneck for model uploading due to limited communication resources \cite{shahid2021communication,lim2020federated}. Conventional FL systems allocate orthogonal channels to the devices for model uploading, which, however, leads to an explosive bandwidth requirement and fails to support simultaneous transmissions of many devices. To alleviate this bottleneck, over-the-air computation (AirComp) \cite{ota}, which leverages the superposition property of radio waves to support the concurrent transmission from multiple devices, has been introduced to FL \cite{zhu2019broadband}.
By proper pre-equalization and transmit power control at the devices, the server can estimate a linear combination of signals from the devices.
The required bandwidth or communication latency of AirComp is independent of the number of devices and thus significantly relieves the communication bottleneck.
Since the server only requires a weighted sum of local model updates to devise the global model, AirComp suits FL aggregation well, which promotes the emergence of a new area, namely \textit{over-the-air FL} \cite{zhu2019broadband,yang2020federated,lin2022relay}.

However, over-the-air FL is prone to communication distortion during model aggregation due to wireless channel fading and noise \cite{yang2020federated,lin2022relay,9459539,fan2021joint}. The distortion is quantified by the mean square error (MSE) between the received and ground-truth signals. This can mislead the update of the global model and severely degrade the training performance.
To mitigate these drawbacks, various methods have been proposed, such as power control \cite{fan2021joint,zhang2021gradient,cao2022jsac} and transceiver designs \cite{zhu2019broadband,lin2022relay}.
The power control method with a channel inversion-based structure \cite{fan2021joint} is commonly used in over-the-air FL, where power control and denoising factors are co-designed to achieve alignments among devices given individual average transmit power constraints.
An improvement to this design is presented in \cite{zhang2021gradient} by taking model update statistics into account, which accelerates the convergence speed.
Similarly, Cao \textit{et al.} \cite{cao2022jsac} maximized the convergence speed by designing a power control policy.
These approaches, however, assume the participation of all the devices including those with poor channel conditions, which leads to performance depression.
In AirComp, devices with better channel conditions must reduce their transmit power to achieve global alignment across all devices. Consequently, devices with weaker channels tend to dominate the overall communication distortion incurred during alignment. This negative effect can be mitigated through carefully scheduling devices based on channel quality to minimize the impact of weak signals.
In this regard, Zhu \textit{et al.} \cite{zhu2019broadband} proposed a truncated channel inversion strategy, where devices with poor channel conditions are excluded from the concurrent transmission.
While being effective in reducing communication distortion, this approach may exclude devices with important model updates that could accelerate training. To maximize the benefits of collaborative FL, device scheduling must take into account both channel quality and the significance of local updates. Simply scheduling devices based on channel conditions alone could leave out local updates that are most informative for improving the global model. Therefore, an optimal scheduling policy for over-the-air FL is needed to not only minimize communication distortion but also optimize learning performance by judiciously selecting devices according to channel quality and the importance of their local model updates.

Recent works have proposed different device scheduling policies for FL to alleviate the distortion of AirComp. For instance, Yang \textit{et al.} \cite{yang2020federated} proposed a joint device selection and beamforming design for over-the-air FL, which prioritizes devices with better channel quality.
An update importance aware scheduling policy was developed in \cite{ma2021user} that selects devices with more informative model updates, i.e., those with larger update norms \cite{luping2019cmfl}, for faster convergence in over-the-air FL.
Nevertheless, these policies considered a single aspect for device scheduling, i.e., either the channel condition or update importance, which limits the training performance in extreme cases.
To address this issue, a recent work \cite{du2023gradient} explored a mixed scheduling policy that defines device quality as a weighted sum of its local update norm and channel condition.
Moreover, some works \cite{scheduling1,scheduling2} proposed comprehensive device scheduling designs to balance the training performance and energy cost via Lyapunov optimization.
By considering both update importance and channel conditions in device scheduling, these works significantly improved the training performance of over-the-air FL.
However, they adopted deterministic metrics for device scheduling, which may lead to biased model aggregation in each communication round.
Specifically, the devices in FL have non-independent and identically distributed (non-IID) data, and the local training on these data can result in diverged local updates.
Consequently, the global update aggregated from a subset of selected devices may deviate from the expected one from all the devices. \footnote{We note that even in the IID scenarios, variations exist in data samples observed by the devices \cite{attota2021ensemble}. However, in this work, our focus is specifically on ensuring the statistical unbiasedness of the model update with respect to the non-IID data distribution among devices.}
The biased update degrades the convergence performance of global training, as it is only guaranteed to converge to a neighbor of the optimum \cite{wang2022client,wu2022incentive}.
These limitations highlight the necessity for novel approaches in device scheduling that can ensure the unbiased aggregation while utilizing only a subset of local updates.

To ensure unbiased model aggregations in FL, probabilistic device scheduling has been explored in the existing literature \cite{renjinke,zhang2022communication,sun2022stochastic}.
By reweighting the local gradient of each selected device with a scheduling probability, the aggregated gradient provides an unbiased estimate of the global gradient. This approach effectively avoids misleading the update of global model training and improves the learned model accuracy.
Ren \textit{et al.} \cite{renjinke} proposed a probabilistic scheduling framework for FL that admits devices based on an optimized probability distribution, effectively balancing the update importance and per-round communication latency. Based on the convergence analysis, this framework was further developed in \cite{zhang2022communication} to minimize the total communication time of an FL process.
Moreover, probabilistic device scheduling was adopted in \cite{sun2022stochastic} to construct an unbiased estimate of the global update computed on all the devices by involving only a subset of fast devices.
However, these scheduling methods mainly focused on reducing the communication latency for conventional FL systems.
In contrast, the communication latency for AirComp is determined by the number of transmitted symbols, which is independent of the device scheduling.
Consequently, these methods are not applicable to over-the-air FL.
In this context, designing a probabilistic device scheduling policy for over-the-air FL poses new challenges. 
Specifically, the policy must jointly consider the impact of channel conditions on communication distortion and the importance of local updates uploaded by the devices. Therefore, developing a novel device scheduling policy becomes necessary to enhance the learning performance of over-the-air FL.
To the best of our knowledge, we are the first to investigate probabilistic scheduling in over-the-air FL.

In this work, we propose a probabilistic device scheduling policy, which strikes a balance between the communication distortion and gradient importance, to improve the training performance of over-the-air FL.
Our contributions are summarized as follows:
\begin{itemize}
    \item We propose a probabilistic device scheduling framework for over-the-air FL named \textit{PO-FL}.
    Specifically, in each communication round, the devices are scheduled for AirComp according to certain probabilities and their local model updates are reweighted using the corresponding scheduling probabilities in aggregation.
    The server constructs an estimate of the weighted sum of the received local updates to update the global model, which is sent back to the devices for the next round of training.
    \item We prove that the constructed global update at the server is an unbiased estimate of the desired weighted sum of local updates from all the devices.
    Then, we analyze the convergence of PO-FL for both convex and non-convex loss functions to characterize the impact of device scheduling. The analytical results show that the device scheduling critically affects the training performance through both the communication distortion and global update variance. 
    \item  Based on the analysis, we investigate the optimization of device scheduling to improve the training performance of PO-FL. We propose a channel and gradient-importance aware device scheduling algorithm to jointly minimize the communication distortion and global update variance.
    The proposed algorithm considers both channel conditions and gradient importance of devices and assigns proper weights to the devices.
    \item We evaluate the proposed device scheduling policy on two image classification datasets, i.e., MNIST and CIFAR-10, via extensive simulations. The simulation results demonstrate the benefits of the proposed scheduling policy in improving the learning performance. Specifically, the proposed policy achieves faster convergence than the baseline methods and performs close to the idealized case without channel noise.
\end{itemize}

The rest of this paper is organized as follows.
In Section \ref{sec:model}, we describe the system model.
In Section \ref{po-fl}, we present the training process of PO-FL and analyze its convergence.
We formulate the device scheduling problem based on the analytical results and propose a channel and gradient-importance aware algorithm to solve the problem in Section \ref{sec:optimization}.
In Section \ref{sec:simulation}, we evaluate the proposed design via extensive simulations. Finally, we conclude the paper and discuss the future works in Section \ref{sec:conclusion}.

\emph{Notations:}
Throughout this paper, we use boldface lower-case letters (e.g., $\bm{x}$) and calligraphy letters (e.g., $\mathcal{S}$) to represent vectors and sets, respectively. We use $\bm{x}[d]$ and $\| \bm{x} \|_2$ to denote the $d$-th entry and $l_2$-norm, respectively, of vector $\bm{x}$, and $|\mathcal{S}|$ to denote the cardinality of set $\mathcal{S}$.
We use $\bm{0}_D$, $\mathbf{I}_D$, and $\bm{1}_D$ to denote the $D \times D$ zero matrix, $D \times D$ identity matrix, and $D$-dimensional all-ones vector, respectively. 
The set of integers $\{0,1,\dots,T-1\}$ is denoted by $[T]$.
In addition, $\mathbbm{1}\{\cdot\}$ is the indicator function, i.e., $\mathbbm{1}\{A\}=1$ if event $A$ happens and $\mathbbm{1}\{A\}=0$ otherwise, and $\mathcal{CN}(\mu,\sigma^2)$ represents the distribution of a circularly symmetric complex Gaussian random variable with mean $\mu$ and variance $\sigma^2$.

\section{System Model}\label{sec:model}

\subsection{Federated Learning}\label{sec:2a}

We consider an FL system consisting of a server and $N$ devices, all of which are equipped with a single antenna.
Let $\mathcal{N} \triangleq \{1,\dots, N\}$ denote the set of devices.
Each device $i\in \mathcal{N}$ has a local dataset $\mathcal{D}_i = \{(\bm{u}_{i,j}, v_{i,j})\}_{j=1}^{m_i}$ consisting of $m_i = |\mathcal{D}_i| $ training data samples, where $\bm{u}_{i,j}$ is the feature of sample $j$ at device $i$ and $v_{i,j}$ is the label of $\bm{u}_{i,j}$.
The devices collaborate to train a model $\bm{w} \in \mathbb{R}^{D}$ with $D$ trainable parameters under the coordination of the server.
The training objective is the loss over all the data samples, which is given as follows:
\begin{equation}
    \min_{\bm{w}\in \mathbb{R}^{D}} f(\bm{w}) \triangleq \frac{1}{M} \sum_{i\in\mathcal{N}} \sum_{j=1}^{m_i} F_i(\bm{w}; \bm{u}_{i,j}, v_{i,j}),
    \label{eq:obective}
\end{equation}
where $M \triangleq \sum_{i \in \mathcal{N}} m_i$ and $F_i(\bm{w}; \bm{u}_{i,j}, v_{i,j})$ is the training loss on sample $(\bm{u}_{i,j}, v_{i,j})$.

To solve the problem in \eqref{eq:obective}, we adopt the classic FedAvg algorithm \cite{fedavg}. The training objective is decomposed as follows:
\begin{equation}
    f(\bm{w}) = \sum_{i\in\mathcal{N}} \rho_i^t f_i(\bm{w}),
\end{equation}
where $\rho_i^t \triangleq \frac{m_i}{M}$ denotes the aggregation weight of device $i$ and
\begin{equation}
    f_i(\bm{w}) \triangleq \frac{1}{m_i} \sum_{j=1}^{m_i} F_i(\bm{w}; \bm{u}_{i,j}, v_{i,j})
\end{equation}
denotes the local loss function at device $i$. Accordingly, FedAvg solves Problem \eqref{eq:obective} in a distributed manner by minimizing the local loss function $f_{i}\left(\bm{w}\right)$ at each device iteratively. Specifically, the training process of FedAvg spans $T$ communication rounds and the $t$-th communication round consists of the following main steps:
\begin{enumerate}
    \item \textbf{Global model broadcasting}: The server broadcasts the updated global model $\bm{w}^t$ to all the devices.
    
    \item \textbf{Local gradient calculation}: 
    Each device $i$ calculates the local gradient by randomly sampling a mini batch of samples $\bm{\xi}_i$ from its local dataset $\mathcal{D}_i$ as $\bm{g}_i^t \triangleq \nabla f_i(\bm{w}^t; \bm{\xi}_i)$.
    
    \item \textbf{Gradient uploading}: Due to the limited communication resources, the server selects a subset of devices $\mathcal{S}^t$ according to a specific criterion, and each selected device uploads its local gradient to the server.
    
    \item \textbf{Model aggregation}: The server estimates the weighted sum of the uploaded local gradients $\bm{y}^t \!=\! \sum_{i \in \mathcal{S}^t} \rho_i^t \bm{g}_i^t$ as $\hat{\bm{y}}^t$ and updates the global model as follows:
    \begin{equation}
       \bm{w}^{t+1} = \bm{w}^{t} - \eta^t \hat{\bm{y}}^t,
       \label{eq:new-model}
    \end{equation}
    where $\eta^t>0$ is the learning rate.
\end{enumerate}

\subsection{Over-the-Air Federated Learning}\label{sec:2b}

The uplink bandwidth cost is the communication bottleneck of FL, since the training process requires periodic uploading of local gradients. To relieve this bottleneck, we adopt the AirComp technique \cite{ota} for uplink model aggregation \cite{yang2020federated}, which allows multiple devices to concurrently upload their gradients over the same channel and the server can reconstruct $\sum_{i \in \mathcal{S}^t} \rho_i^t \bm{g}_i^t$ with appropriate signal processing.

As the entries in each local gradient may vary significantly in value, the selected devices need to transform the gradients into normalized signal vectors to facilitate pre-equalization and power control \cite{zhu2019broadband,yang2020federated,lin2022relay}.
To this end, each selected device $i \in \mathcal{S}^t$ normalizes its local gradient $\bm{g}^{t}_{i}$ into a symbol vector $\bm{s}^{t}_{i}\in\mathbb{C}^{D}$ such that $\mathbb{E}[\bm{s}^{t}_{i} (\bm{s}^{t}_{j})^{\HH}] = \bm{0}_D$ and $\mathbb{E}[\bm{s}^{t}_{i} (\bm{s}^{t}_{i})^{\HH}] = \mathbf{I}_D$, $\forall i,j\in\mathcal{S}^t,i\neq j$.\footnote{Similar to \cite{zhu2019broadband,yang2020federated,lin2022relay}, we assume that the local gradients of the devices are independent and identically distributed.}
Specifically, each device first computes the mean and variance of the local gradient as $M_{i}^{t} \triangleq \frac{1}{D} \sum_{d=1}^{D} \bm{g}_{i}^{t}[d]$ and $V_{i}^{t} \triangleq \frac{1}{D} \sum_{d=1}^{D} (\bm{g}_{i}^{t}[d] - M_{i}^{t} )^2$, respectively.
The server collects the local statistics $\{M_{i}^{t}, V_{i}^{t}\}$ through uplink control channels and computes the global mean and variance as $M_{\bm{g}}^{t} \triangleq \sum_{i\in \mathcal{S}^t} \rho_{i}^{t} M_{i}^{t}$ and $V_{\bm{g}}^{t} \triangleq \sum_{i\in \mathcal{S}^t} \rho_{i}^{t} V_{i}^{t}$, respectively.
Note that the local statistics $\{M_{i}^{t}, V_{i}^{t}\}$ are scalars and can be uploaded to the server with a negligible communication cost.
Then, the server broadcasts the global statistics $\{M_{\bm{g}}^{t}, V_{\bm{g}}^{t}\}$ to the selected devices. 
Afterward, each device $i$ normalizes its gradient as a symbol vector $\bm{s}^{t}_{i}$:
\begin{equation}
    \bm{s}^{t}_{i} =\frac{1}{\sqrt{V_{\bm{g}}^{t}}} \left( \bm{g}^{t}_{i} - M_{\bm{g}}^{t} \bm{1}_D \right).
    \label{eq:normalize}
\end{equation}

By applying transmit scaling, the selected devices concurrently upload the scaled $D$-dimensional symbol vectors $\{\bm{s}_i^t: \forall i \in \mathcal{S}^t\}$ to the server. 
Let $h_i^t \in \mathbb{C}$ denote the channel coefficient between device $i$ and the server, and $b^{t}_{i}$ denote the transmit scalar of device $i$ in the $t$-th communication round.
We assume an individual transmit power constraint at each device, i.e.,
\begin{equation}
    \mathbb{E} \left[ |b_i^t \bm{s}_i^t[d] |^2 \right] = |b_i^t|^2 \leq P, \forall 1 \leq d \leq D,
    \label{eq:power}
\end{equation}
where $P$ is the maximum transmit power. 
The received signal vector at the server is given by:
\begin{equation}
    \Tilde{\bm{y}}^t = \sum_{i\in \mathcal{S}^t} h_{i}^{t} b^{t}_{i} \bm{s}^{t}_{i} + \bm{z}^t,
    \label{eq:yt}
\end{equation}
where $\bm{z}^t \in \mathbb{C}^D$ is the additive channel noise vector with each element $\bm{z}^t[d] \sim \mathcal{CN}(0, \sigma_z^2), 1 \leq d \leq D$.

To alleviate the negative effect of the channel noise on signal estimation, the server employs a de-noising receive scalar $a^t$ to the received signal vector $\Tilde{\bm{y}}^t$ and obtains $\frac{1}{a^t} \Tilde{\bm{y}}^t$. 
Then, the server can use the normalization factors $M_{\bm{g}}^{t}$ and $V_{\bm{g}}^{t}$ to construct the gradient estimate $\hat{\bm{y}}^t$ via a de-normalization step as follows:
\begin{align}
    \hat{\bm{y}}^t =& \sqrt{V_{\bm{g}}^t} \frac{1}{a^t} \Tilde{\bm{y}}^t + M_{\bm{g}}^t \mathbf{1}_D \\
    =& \sum_{i\in \mathcal{S}^t} \frac{h_{i}^{t} b^{t}_{i}}{a^{t}} \bm{g}^{t}_{i} + \frac{ \sqrt{V_{\bm{g}}^{t}}}{a^{t}} \bm{z}^t.
    \label{eq:eq6}
\end{align}
We illustrate the process of AirComp in an FL system in Fig. \ref{fig:computation}.
The accuracy of the estimation is quantified by the MSE between the constructed estimate $\hat{\bm{y}}^t$ and the ground truth $\bm{y}^t = \sum_{i \in \mathcal{S}^t} \rho_i^t \bm{g}_i^t$ as follows:
\begin{align}
    e_{\text{com}}^{t} \triangleq & \mathbb{E} \left[ \left\| \hat{\bm{y}}^t - \bm{y}^t \right\|_2^2 \right] 
    \label{eq:mse-eq1}
    \\
    = & \mathbb{E} \left[ \left\| \sum_{i\in \mathcal{S}^t} \frac{h_{i}^{t} b^{t}_{i} }{a^{t}} \bm{g}^{t}_{i} + \frac{ \sqrt{V_{\bm{g}}^{t}}}{a^{t}} \bm{z}^t - \sum_{i \in \mathcal{S}^t} \rho_i^t \bm{g}_i^t \right\|_2^2 \right].
    \label{eq:mse-eq2}
\end{align}

\begin{figure}[!t]
    \centering
    \includegraphics[width=\columnwidth]{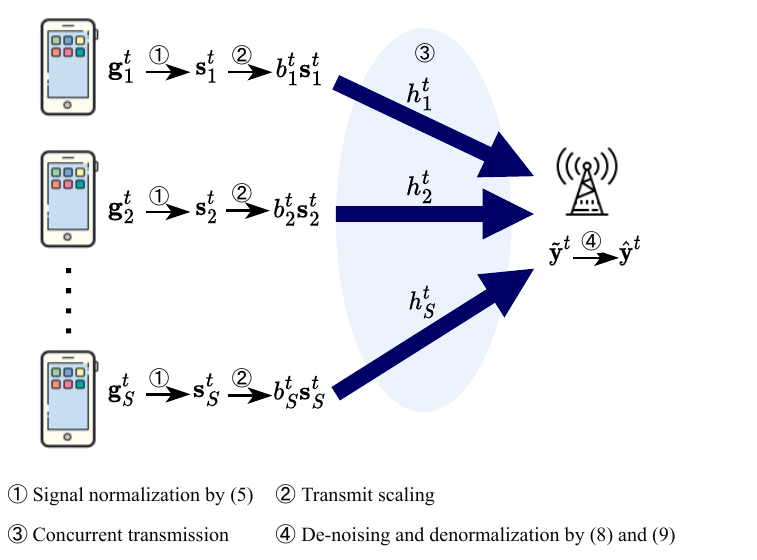}
    \caption{An illustration of AirComp in an FL system, where $|\mathcal{S}^t| = S$ devices are assumed to be scheduled for gradient uploading.}
    \label{fig:computation}
\end{figure}

To maximize the estimation accuracy, we minimize the communication distortion defined in \eqref{eq:mse-eq2} by optimizing the transmit scalars $\{ b_i^t \}$ and the de-noising receive scalar $a^t$.
The optimal transceiver design variables and the minimum communication distortion are given in the following lemma.
\begin{lemma} \label{lem:transceiver}
With power constraints in \eqref{eq:power}, the optimal transceiver design variables are given by
\begin{equation}
    b^{t}_{i} = \frac{1}{h_{i}^{t}} \rho_{i}^{t} a^t, \forall i \in \mathcal{S}^t,
\end{equation}
and
\begin{equation}
    a^t = \min_{i\in \mathcal{S}^t} \frac{\sqrt{P} |h_{i}^{t}|}{\rho_{i}^{t}}. 
\end{equation}
Thus, we can rewrite the constructed gradient estimate as follows:
\begin{equation}
	\hat{\bm{y}}^t = \sum_{i\in \mathcal{S}^t} \rho^t_i \bm{g}^{t}_{i} + \frac{ \sqrt{V_{\bm{g}}^{t}}}{a^{t}} \bm{z}^t,
    \label{eq:yt-eq1}
\end{equation}
and the minimum communication distortion is given by:
\begin{equation}
    e_{\text{com}}^{t} = \mathbb{E} \left[ \left\| \frac{ \sqrt{V_{\bm{g}}^{t}}}{a^{t}} \bm{z}^t \right\|_2^2 \right] = \frac{D \sigma_z^2 V_{\bm{g}}^{t}}{P} \max_{i\in \mathcal{S}^t} \frac{(\rho_{i}^{t})^2}{|h_{i}^{t}|^2}.
    \label{eq:mse-detail}
\end{equation}
\end{lemma}
\begin{proof}
    The proof is similar to that of \cite[Lemma~1]{lin2022relay} and is omitted for brevity.
\end{proof}

We see from \eqref{eq:mse-detail} that the estimation accuracy depends on the set of devices $\mathcal{S}^t$ that are selected. 
To avoid significant model aggregation distortion, existing works employ various methods such as power control \cite{fan2021joint,zhang2021gradient,cao2022jsac} and transceiver designs \cite{zhu2019broadband,lin2022relay}.
These works allow full participation of all clients including the devices with poor channels, which still deteriorates the training performance.
One straightforward method is to only schedule devices with acceptable channel conditions for model aggregation.
However, as will be elaborated in Section \ref{sec:optimization}, this scheduling policy may exclude the important updates of those weak devices.
In addition, aggregating the local gradients from only a fraction of devices usually results in a biased estimate of the desired global gradient $\nabla f(\bm{w}^t) = \sum_{i\in \mathcal{N}} \frac{m_i}{M} \bm{g}^{t}_{i}$, which significantly slows down the convergence speed and compromises the learned model accuracy \cite{wang2022client,wu2022incentive}.
To achieve unbiased gradient estimates, we propose a probabilistic device scheduling scheme for over-the-air FL named PO-FL.
It is worth noting that existing works \cite{renjinke,zhang2022communication} on probabilistic device scheduling aim to minimize the communication time in conventional FL with digital communication scheme. However, the communication time is a constant in AirComp due to the concurrent transmission, which makes the designs in these works inapplicable to over-the-air FL. In contrast, the proposed device scheduling scheme focuses on alleviating the model aggregation distortion in over-the-air FL. 
In the next section, we describe the PO-FL framework in detail and analyze its convergence performance.

\section{The PO-FL Framework}\label{po-fl}

In this section, we elaborate on the PO-FL framework and characterize the impact of device scheduling on the convergence for both convex and non-convex loss functions.

\subsection{Framework Description}

In the $t$-th communication round, the server computes the scheduling probability $p_i^t$ for each device $i$ based on an optimization algorithm, which will be elaborated in Section \ref{sec:optimization}. 
The server then selects a subset of devices $\mathcal{S}^t$ from the device set $\mathcal{N}$, choosing device $i$ with probability $p_i^t$.
The selected devices concurrently upload the local gradients $\{ \bm{g}^t_i \}$ to the server for aggregation following the AirComp procedure elaborated in Section \ref{sec:2b}.
To achieve an unbiased model aggregation, the local gradient of each selected device $i \in \mathcal{S}^t$ is reweighted by the reciprocal of its scheduling probability $p_i^t$, i.e., $\rho_i^t = \frac{m_i}{M p_i^t}$.
Thus, we obtain $a^t = \min_{i\in \mathcal{S}^t} \frac{M p_i^t \sqrt{P} |h_{i}^{t}|}{m_i}$, and the constructed gradient in \eqref{eq:yt-eq1} becomes:
\begin{equation}
    \hat{\bm{y}}^t = \sum_{i\in \mathcal{S}^t} \frac{m_i}{M p_i^t} \bm{g}^{t}_{i} + \max_{i\in \mathcal{S}^t} \frac{ m_i \sqrt{ V_{\bm{g}}^{t}}}{M p_i^t \sqrt{P} |h_{i}^{t}|} \bm{z}^t.
    \label{eq:yt-2}
\end{equation}
As will be proved in Lemma \ref{unbiased}, the constructed gradient in \eqref{eq:yt-2} is an unbiased estimate of the desired global gradient $\nabla f(\bm{w}^t)$.
Then, the global model is updated by \eqref{eq:new-model} for the next round of training. The above process repeats for $T$ communication rounds.
We summarize the training process in Algorithm \ref{alg} and provide an overview of the PO-FL framework in Fig. \ref{fig:framework}.

Note that the estimation of the global gradient in \eqref{eq:yt-2} critically depends on the scheduled device set $\mathcal{S}^t$, which is determined by the scheduling probabilities $\{p_i^t\}$.
In particular, the probabilities should be closely related to the device characteristics including the local gradients $\{\bm{g}_i^t \}$, the data amounts $\{ m_i \}$, and the channel coefficients $\{ h_i^t \}$.
To obtain accurate estimates, we need to optimize the scheduling probabilities $\{p_i^t\}$ by jointly considering these factors.
In the following subsection, we will analyze the convergence of PO-FL to gain a more in-depth understanding of the impact of device scheduling. 
Then, we will formulate the problem of optimizing the device scheduling based on the analytical results and propose an efficient algorithm to solve this problem in Section \ref{sec:optimization}.

\begin{figure*}
    \centering
    \includegraphics[width=\textwidth]{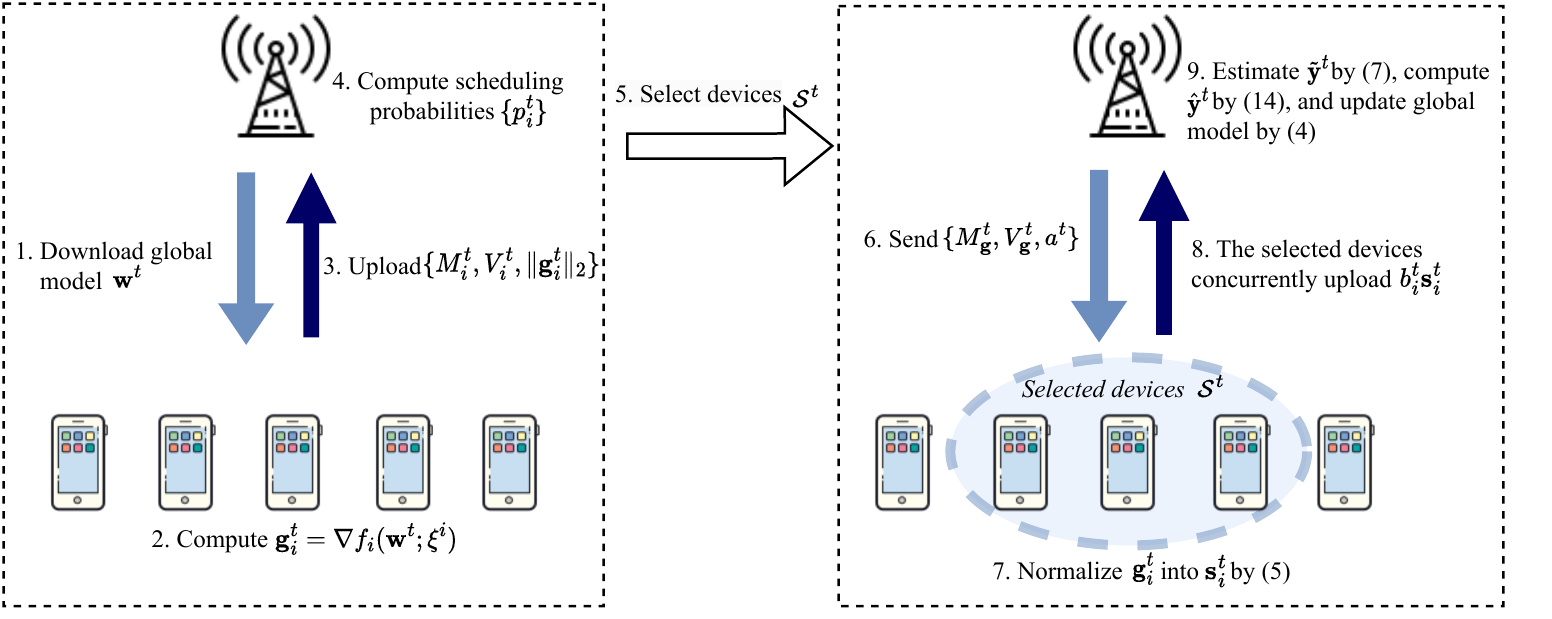}
    \caption{
    An overview of the PO-FL framework.
    Compared with \cite{zhu2019broadband,yang2020federated,lin2022relay}, in Step 3, we additionally require the devices to upload the gradient norm $\|\bm{g}_i^t\|_2$ for the computation of the scheduling probabilities in \eqref{eq:single}.
    Note that $\|\bm{g}_i^t\|_2$ is a scalar and can be uploaded to the server with a negligible communication cost.
    }
    \label{fig:framework}
\end{figure*}

\begin{algorithm}[tb]
\caption{Training Process of PO-FL} \label{alg}
\begin{algorithmic}[1]
\STATE{Initialize a global model $\bm{w}^{0}$;}
\FOR{$t=0,1,\dots,T-1$}
    \STATE{Broadcast the global model $\bm{w}^{t}$ to each device $i\in \mathcal{N}$;}
    \FOR{each device $i\in \mathcal{N}$}
    \STATE{Compute the local gradient with respect to the local mini batch $\bm{\xi}_i \subseteq \mathcal{D}_i$, i.e., $\bm{g}_i^t \triangleq \nabla f_i(\bm{w}^t; \bm{\xi}_i)$;}
    \STATE{Compute and upload the local statistics $\{M_i^t, V_i^t\}$ and gradient norm $\|\bm{g}_i^t\|_2$ to the server;}
    \ENDFOR
    \STATE{The server computes the scheduling probabilities $\{p_i^t\}$, selects a subset of devices $\mathcal{S}^t$ according to $\{p_i^t\}$, and broadcasts $\left\{M_{\bm{g}}^t, V_{\bm{g}}^t, a^t\right\}$ to selected devices;}
    \FOR{each device $i\in \mathcal{S}^t$}
    \STATE{Normalize $\bm{g}_i^t$ into $\bm{s}_i^t$ by \eqref{eq:normalize};}
    \ENDFOR
    \STATE{The devices in $\mathcal{S}^t$ concurrently upload $\{b_i^t \bm{s}_i^t\}$ to the server via AirComp;}
	\STATE{The server receives the signal vector by \eqref{eq:yt}, estimates the global gradient by \eqref{eq:yt-2}, and updates the global model by \eqref{eq:new-model};}
\ENDFOR
\end{algorithmic}
\end{algorithm}

\subsection{Convergence Analysis}\label{sec:convergence}

To facilitate the analysis, we make the following assumptions on the local loss functions \cite{amiri2021convergence,lin2022coexisting,sun2021semi, chen2023knowledge,liconvergence}.

\begin{assumption}\label{smooth}
($L$-smoothness)
There exists a constant $L>0$ such that for any $\bm{w}_1,\bm{w}_2\in\mathbb{R}^{D}$, we have:
\begin{equation}
    \| \nabla f_i(\bm{w}_1) - \nabla f_i(\bm{w}_2) \|_2 \leq L \| \bm{w}_1 - \bm{w}_2 \|_2, \forall i\in\mathcal{N}.
\end{equation}
\end{assumption}

\begin{assumption}\label{sgd}
(Unbiased and variance-bounded gradient)
Any stochastic gradient $ \nabla f_i(\bm{w}; \bm{\xi}_i)$ computed on a randomly sampled mini batch $\bm{\xi}_i$ is an unbiased estimate of the gradient over all the data of device $i$, i.e., 
\begin{equation}
    \mathbb{E}\left[ \nabla f_i(\bm{w}; \bm{\xi}_i) \right] = \nabla f_i(\bm{w}), \forall i\in\mathcal{N}.
\end{equation}
Besides, there exists a constant $\sigma>0$ such that 
\begin{equation}
    \mathbb{E}\left[ \| \nabla f_i(\bm{w}; \bm{\xi}_i) - \nabla f_i(\bm{w}) \|_2^2 \right] \leq \sigma^2, \forall i\in\mathcal{N}.
\end{equation}
\end{assumption}

\begin{assumption}\label{sgd2}
(Bounded gradient)
There exists a constant $G>0$ such that
\begin{equation}
    \| \nabla f_i(\bm{w}; \bm{\xi}_i) \|^2 \leq G^2, \forall i\in\mathcal{N}.
\end{equation}
\end{assumption}

The following lemma demonstrates the unbiasedness of gradient estimate $\hat{\bm{y}}^t$ constructed in the PO-FL framework.

\begin{lemma}\label{unbiased}
The constructed global gradient in each communication round is an unbiased estimate of the global gradient, i.e.,
\begin{equation}
    \mathbb{E} \left[\hat{\bm{y}}^{t} \right] = \nabla f(\bm{w}^{t}),
    \label{eq:unbised}
\end{equation}
where the expectation is taken with respect to the noise distribution, mini-batch data sampling, and device scheduling.
\end{lemma}
\begin{proof} 
We derive the result in \eqref{eq:unbised} as follows:
\begin{equation}
    \mathbb{E} \left[ \hat{\bm{y}}^{t} \right] 
    \overset{(\text{a})}{=} \mathbb{E} \left[ \sum_{i\in \mathcal{S}^t} \frac{m_i}{M p_i^t} \bm{g}^{t}_{i} \right] 
    \overset{(\text{b})}{=} \sum_{i\in \mathcal{N}} \frac{m_i}{M} \mathbb{E} \left[ \bm{g}^{t}_{i} \right] 
    \overset{(\text{c})}{=} \nabla f(\bm{w}^{t}),
\label{eq:app-lemma-1}
\end{equation}
where (a) is due to the zero-mean of the channel noise, i.e., $\mathbb{E}[\bm{z}^t]=0$, (b) follows from that:
\begin{equation}
    \mathbb{E} \left[ \frac{m_i}{M p_i^t} \bm{g}^{t}_{i} \right] = \sum_{j\in \mathcal{N}} p_i^t \frac{m_j}{M p_i^t} \mathbb{E} \left[\bm{g}^{t}_{j}\right] = \sum_{j\in \mathcal{N}} \frac{m_j}{M} \mathbb{E} \left[ \bm{g}^{t}_{j} \right], \forall i\in \mathcal{N},
\end{equation}
and (c) is from Assumption \ref{sgd}.
\end{proof}

Let $\bm{w}^*$ denote the global optimum of $f(\cdot)$, i.e., $\bm{w}^* = \arg\min_{\bm{w}\in\mathbb{R}^D} f(\bm{w})$. 
The local objectives of many ML problems, such as those involving convolutional neural networks (CNNs), are non-convex.
For these problems, the learned model may converge to a local minimum or a saddle point.
Following \cite{bottou2018optimization,wang2021cooperative,barakat2020convergence}, we consider an algorithm to have achieved convergence if it converges to a stationary point of the global loss function, i.e., if its expected squared gradient norm $\min_{t\in[T]} \mathbb{E}\left[ \left\| \nabla f(\bm{w}^t) \right\|_2^2 \right]$ is zero.
In the following theorem, we show the convergence of the proposed PO-FL framework for general ML problems, including those with non-convex local objectives.

\begin{theorem}\label{thm:non-convex}
Let $\gamma_T \triangleq \sum_{t=0}^{T-1} \eta^t$.
With Assumptions \ref{smooth}-\ref{sgd2} and Lemma \ref{unbiased}, if the learning rates $\eta^t$ satisfy $\eta^t \leq \frac{1}{L}, \forall t\in [T]$, we have:
\begin{equation}
\begin{split}
    &\min_{t\in[T]} \mathbb{E}\left[ \left\| \nabla f(\bm{w}^t) \right\|_2^2 \right] \\
    \leq & \frac{2}{\gamma_T}  \left( \mathbb{E}[f(\bm{w}^{0})] - f(\bm{w}^{*}) \right) + \frac{L}{\gamma_T}  \sum_{t=0}^{T-1} (\eta^t)^2 \left( 1+\frac{1}{\alpha} \right) \sigma^2 \\
    & + \frac{L}{\gamma_T}  \sum_{t=0}^{T-1} (\eta^t)^2 \left[ (1+\alpha) \mathbb{E}[e_{\text{com}}^{t}] + \left( 1+\frac{1}{\alpha} \right) \mathbb{E}[e_{\text{var}}^{t}] \right],
\end{split}
\label{eq:non-convex}
\end{equation}
where $\alpha>0$ determines the weight between the communication distortion and global update variance, $e_{\text{com}}^{t}$ is the communication distortion given in \eqref{eq:mse-detail}, $e_{\text{var}}^{t} \triangleq \left\| \sum_{i\in \mathcal{S}^t} \frac{m_i}{M p_i^t} \bm{g}^{t}_{i} \right.$ $\left.- |\mathcal{S}^t| \sum_{j\in \mathcal{N}} \frac{m_j}{M} \bm{g}^{t}_{j} \right\|_2^2$ is the global update variance, and the expectations are taken with respect to the noise distribution, mini-batch data sampling, and device scheduling.
\end{theorem}
\begin{proof}
Please refer to Appendix \ref{appendix-non-convex}.
\end{proof}

In \eqref{eq:non-convex}, $\alpha$ balances the weight between the communication distortion and global update variance.
Specifically, a larger value of $\alpha$ places greater emphasis on reducing the communication distortion rather than global update variance, while a smaller $\alpha$ places more weight on preserving important gradient vectors.
The selection of $\alpha$ depends on several factors, including the specific training task, communication conditions, and device capabilities.
While it is difficult to provide a universal guideline for selecting the optimal value of $\alpha$, we will analyze the effect of $\alpha$ on training performance via simulations in Section \ref{sec:simulation} and provide practical insight into selecting an appropriate $\alpha$.

From Theorem \ref{thm:non-convex}, we obtain an upper bound for $\min_{t\in[T]} \mathbb{E}\left[ \left\| \nabla f(\bm{w}^t) \right\|_2^2 \right]$, which increases with both the communication distortion $\mathbb{E}[e_{\text{com}}^{t}]$ and the global update variance $\mathbb{E}[e_{\text{var}}^{t}]$.
According to \cite{bottou2018optimization}, the larger this upper bound is, the more communication rounds are required for convergence. Therefore, we can minimize both factors in order to accelerate the learning convergence.
Based on Theorem \ref{thm:non-convex}, we can further characterize the convergence behavior of PO-FL with non-convex local loss functions in the following corollary.

\begin{corollary}\label{corollary2}
If the learning rates satisfy $\lim_{T\rightarrow \infty} \gamma_T = \infty$ and $\lim_{T\rightarrow \infty} \sum_{t=0}^{T-1} (\eta^t)^2 < \infty$, the right-hand-side (RHS) of \eqref{eq:non-convex} converges to zero as $T \rightarrow \infty$.
\end{corollary}
\begin{proof}
Please refer to Appendix \ref{proof-corollary2}.
\end{proof}

From Corollary \ref{corollary2}, we see that PO-FL with non-convex loss functions can converge to a stationary point of the global loss function. However, Corollary \ref{corollary2} cannot guarantee that PO-FL converges to the optimum of the global loss function $f(\cdot)$, i.e., $\bm{w}^*$. Next, we further analyze the convergence of PO-FL in the special case with convex loss functions and show that the learned model is guaranteed to converge to the optimal model $\bm{w}^*$.
Specifically, we adopt the best model throughout the training process with $T$ communication rounds, i.e., $\tilde{\bm{w}}^T = \arg\min_{\bm{w}^t,\forall t\in[T]} \mathbb{E} [f(\bm{w}^t)]$, as the output model of PO-FL.
The PO-FL is deemed as converged if $ \mathbb{E} [f(\tilde{\bm{w}}^T)]- f(\bm{w}^{*})$ becomes zero.
To proceed, we make the following additional assumption on the local loss functions.

\begin{assumption}\label{convex}
(Convexity)
For any $\bm{w}_1,\bm{w}_2\in\mathbb{R}^{D}$, we have:
\begin{equation}
    f_i(\bm{w}_1) - f_i(\bm{w}_2) - \langle \nabla f_i(\bm{w}_1), \bm{w}_1 - \bm{w}_2  \rangle \geq 0, \forall i\in\mathcal{N}.
\end{equation}
\end{assumption}

In the following theorem, we derive an upper bound of $\mathbb{E} [f(\tilde{\bm{w}}^T)]- f(\bm{w}^{*})$ under the PO-FL framework.

\begin{theorem}\label{thm:convex}
With Assumptions \ref{smooth}-\ref{convex} and Lemma \ref{unbiased}, we have:
\begin{align}
    & \mathbb{E} [f(\tilde{\bm{w}}^T)] - f(\bm{w}^{*}) \nonumber \\
    \leq & \frac{1}{2\gamma_T} \sum_{t=0}^{T-1} \mathbb{E} \left[ \left\| \bm{w}^{0} - \bm{w}^* \right\|_2^2 \right]  + \frac{1}{2\gamma_T} \sum_{t=0}^{T-1} (\eta^t)^2 G^2 \nonumber \\
    & + \frac{1}{2\gamma_T} \sum_{t=0}^{T-1} (\eta^t)^2 \left[ (1+\alpha) \mathbb{E}[e_{\text{com}}^{t}] + \left( 1+\frac{1}{\alpha} \right) \mathbb{E}[e_{\text{var}}^{t}] \right], \label{eq:convex}
\end{align}
where the expectations are taken with respect to the noise distribution, mini-batch data sampling, and device scheduling.
\end{theorem}
\begin{proof}
Please refer to Appendix \ref{appendix-convex}.
\end{proof}

The following corollary demonstrates that the output model $\tilde{\bm{w}}^T$ in PO-FL with convex loss functions is guaranteed to converge to the optimal model $\bm{w}^{*}$.

\begin{corollary}\label{corollary1}
	For convex loss functions, if the learning rates satisfy $\lim_{T\rightarrow \infty} \gamma_T = \infty$ and $\lim_{T\rightarrow \infty} \sum_{t=0}^{T-1} (\eta^t)^2 < \infty$, the RHS of \eqref{eq:convex} converges to zero as $T \rightarrow \infty$.
\end{corollary}

\begin{proof}
Please refer to Appendix \ref{proof-corollary1}.
\end{proof}

From the above analysis, we see that the PO-FL framework achieves convergence in both convex and non-convex cases. 
Notably, the output model $\tilde{\bm{w}}^{T}$ in the case with convex loss functions can achieve an expected squared gradient norm of zero, i.e., $\mathbb{E}\left[ \left\| \nabla f(\tilde{\bm{w}}^{T}) \right\|_2^2 \right] \rightarrow 0$ when $T\rightarrow \infty$.
Therefore, the analytical results in Theorem \ref{thm:non-convex} and Corollary \ref{corollary2} are also applicable to this case.
Moreover, similar to Theorem \ref{thm:non-convex}, we can conclude from Theorem \ref{thm:convex} that the convergence of PO-FL with convex loss functions is also significantly affected by the communication distortion $\mathbb{E}[e_{\text{com}}^{t}]$ and the global update variance $\mathbb{E}[e_{\text{var}}^{t}]$, which are both dependent on the device scheduling.
This motivates us to consider these two factors as key metrics of the training performance and jointly minimize them by optimizing the device scheduling, as detailed in the next section.

\section{Channel and Gradient-Importance Aware Device Scheduling}\label{sec:optimization}

Based on the analytical results in Section \ref{sec:convergence}, we formulate an optimization problem of device scheduling to improve the training performance as follows:
\begin{align}
    (\text{P1}): \min_{\{p_{i}^{t}\}} & \; (1+\alpha) \mathbb{E}[ e_{\text{com}}^{t} ] + \left( 1+\frac{1}{\alpha} \right) \mathbb{E}[  e_{\text{var}}^{t} ], \label{eq:obj} \\
    \text{s.t. } 
    & 0 < p_{i}^{t} \leq 1, \forall i \in \mathcal{N}.
\end{align}

Note that since the objective \eqref{eq:obj} in Problem (P1) involves expectations with respect to $\mathcal{S}_{t}$, in order to derive an explicit expression, we may first obtain the probability distribution of $\mathcal{S}^{t}$ in terms of $\{p_i^t\}$.
However, as there are $\tbinom{N}{|\mathcal{S}^t|} = \frac{N!}{|\mathcal{S}^t|!(N-|\mathcal{S}^t|)!}$ possible scheduled device sets, calculating the expectations of $e_{\text{com}}^{t}$ and $e_{\text{var}}^{t}$ requires substantial computation overhead. Also, it is hard to verify its convexity so that finding the optimal solution of Problem (P1) is in general NP-hard.

To efficiently solve Problem (P1), we develop a channel and gradient-importance aware device scheduling algorithm in the following, which introduces the cardinality of $\mathcal{S}_{t}$ as a predetermined hyper-parameter. As such, we may optimize $\{p_{i}^{t}\}$ with low complexity and adopt a \emph{sampling without replacement} approach to determine the set of scheduled devices.\footnote{While the determination of $|\mathcal{S}^t|$ is beyond the scope of this paper, in general it should ensure sufficient gradient uploading while avoiding large communication distortion. The impact of $|\mathcal{S}^t|$ will be discussed in Section \ref{sec:simulation}.}
Specifically, we first calculate the optimal scheduling probabilities for single-device scheduling in Section \ref{sec:opt-single}. 
This step serves as an initialization process in the proposed design. 
Subsequently, in Section \ref{sec:opt-multi}, we employ a sampling without replacement approach to schedule the remaining devices. This complements the single-device scheduling and completes the overall scheduling process.

\subsection{Single-Device Scheduling}\label{sec:opt-single}

Suppose that only one device is scheduled in each communication round, i.e., $|\mathcal{S}^t|=1$, which can be expressed using scheduling probabilities as follows:
\begin{equation}
    \mathbb{E}[|\mathcal{S}^t|] = \sum_{i\in \mathcal{N}} p_{i}^{t} = 1.
\end{equation}
The objective in Problem (P1) can be expressed as follows:
\begin{align}
    & (1+\alpha) \mathbb{E}[ e_{\text{com}}^{t} ] + \left( 1+\frac{1}{\alpha} \right) \mathbb{E}[  e_{\text{var}}^{t} ] \nonumber \\
    =& (1+\alpha) \sum_{i=1}^{N} \frac{D \sigma_z^2 \tilde{V}_{\bm{g}}^{t} }{p_{i}^{t} P |h^{t}_{i}|^2} \left(\frac{m_i}{M} \right)^2 \nonumber \\
    & + \left( 1+\frac{1}{\alpha} \right) \sum_{i=1}^{N} \left( \frac{1}{p_{i}^{t}} - 1 \right) \left(\frac{m_i}{M} \right)^2 \| \bm{g}^{t}_{i} \|_2^2,
    \label{eq:obj2}
\end{align}
where $\tilde{V}_{\bm{g}}^{t} = \sum_{i\in \mathcal{N}} \frac{m_i}{M} V_{i}^{t}$.
Therefore, Problem (P1) can be simplified as:

\begin{align}
    (\text{P2}):
    \min_{\{p_{i}^{t}\}} \; &
    (1+\alpha) \sum_{i=1}^{N} \frac{D \sigma_z^2 \tilde{V}_{\bm{g}}^{t} }{p_{i}^{t} P |h^{t}_{i}|^2} \left(\frac{m_i}{M} \right)^2 \nonumber \\
    & + \left( 1+\frac{1}{\alpha} \right) \sum_{i=1}^{N} \left( \frac{1}{p_{i}^{t}} - 1 \right) \left(\frac{m_i}{M} \right)^2 \| \bm{g}^{t}_{i} \|_2^2, \\
    \text{s.t. } & \sum_{i\in \mathcal{N}} p_{i}^{t} = 1, \label{eq:contraint1} \\ 
    & 0 < p_{i}^{t} \leq 1, \forall i \in \mathcal{N}. \label{eq:con2}
\end{align}

Evidently, Problem (P2) is a convex problem. According to the Karush–Kuhn–Tucker (KKT) conditions \cite{boyd2004convex}, we obtain the optimal solution for (P2) as:
\begin{equation}
    p_{i}^{t} 
    = \frac{Q_{i}^{t}}{\sum_{j\in \mathcal{N}} Q^t_j}, \forall i \in \mathcal{N},
\label{eq:single}
\end{equation}
where $Q_i^t$ is defined as
\begin{equation}
    Q_{i}^{t} \triangleq \sqrt{ (1+\alpha) \frac{\tilde{V}_{\bm{g}}^{t} D \sigma_z^2 m_i^2}{P|h^{t}_{i}|^2 M^2} + \left( 1+\frac{1}{\alpha} \right) \frac{m_i^2 \| \bm{g}^{t}_{i} \|_2^2}{M^2}}.
\label{eq:metric}
\end{equation}
Once the scheduling probabilities are computed, the server can schedule one device for this communication round according to \eqref{eq:single}.

\begin{remark}\label{remark}
The scheduling probabilities derived in \eqref{eq:single} achieve a trade-off between reducing the communication distortion and preserving important gradient vectors.
On one hand, $Q_{i}^{t}$ is related to the term $\frac{\tilde{V}_{\bm{g}}^{t} D \sigma_z^2 m_i^2}{P|h^{t}_{i}|^2 M^2}$, which corresponds to the communication distortion $\mathbb{E}[e_{\text{com}}^t]$.
By assigning higher probabilities $\{p_{i}^{t}\}$ to devices with weak channel conditions, the aggregation weights are reduced, thereby avoiding significant communication distortion.
On the other hand, the term $\frac{m_i^2 \| \bm{g}^{t}_{i} \|_2^2}{M^2}$ in $Q_{i}^{t}$ implies that the device with a larger gradient norm should be assigned with a higher probability $p_{i}^{t}$. This ensures that the important gradient vectors are preserved for aggregation.
\end{remark}

\subsection{Multi-Device Scheduling}\label{sec:opt-multi}

To develop a low-complexity solution for Problem (P1), we schedule a set of devices $\mathcal{S}^t$ by sampling a single device without replacement repeatedly for $|\mathcal{S}^t|$ times.
Specifically, we schedule the first device according to the scheduling probabilities $\{p_{i}^{t}\}$ derived in \eqref{eq:single}, and then select the subsequent $|\mathcal{S}^t|-1$ devices successively as follows.

At the $k$-th selection, we have already scheduled $k-1$ devices ($2 \leq k \leq |\mathcal{S}^t|$), and the set of scheduled device indices is denoted by $\{Y_{t, 1},\dots,Y_{t, j},\dots,Y_{t, k-1} \}$, where $Y_{t, j} \in \mathcal{N}$ is the index of the $j$-th scheduled device.
We exclude the devices that are already scheduled and recalculate the scheduling probabilities to satisfy \eqref{eq:contraint1} as follows:
\begin{equation}
    q_{i}^{t} =\left\{\begin{array}{ll} 
    0, & \text { if } i \in\left\{Y_{t,1},\dots,Y_{t,k-1}\right\}, \\
    \frac{p_{i}^{t}}{1-\sum_{j=1}^{k-1} p_{Y_{t,j}}^{t}}, & \text { otherwise. }
    \end{array}\right.
\label{eq:q_i}
\end{equation}
Using the scheduling probabilities $\{q_{i}^{t}\}$ from \eqref{eq:q_i}, we select the $k$-th device $Y_{t, k}$ and include it in $\mathcal{S}^t$.
The above process is repeated until $|\mathcal{S}^t|$ devices are scheduled.
Hence, the aggregated gradient from the devices in $\mathcal{S}^t$ is computed as follows:
\begin{equation}
    \bm{g}^{t} = \frac{1}{|\mathcal{S}^t|} \sum_{k=1}^{|\mathcal{S}^t|} \frac{m_i}{M q^{t}_{Y_{t,k}}} \bm{g}^{t}_{Y_{t,k}}.
    \label{eq:multi}
\end{equation}

It is worth noting that the aggregated gradient in \eqref{eq:multi} is still an unbiased estimate of the global gradient $\nabla f(\bm{w}^{t})$. The proof is similar to that of \cite[Lemma~1]{renjinke} and is omitted for brevity.
Our proposed device scheduling approach takes into account both the channel condition and gradient importance. 
By adopting this approach, we ensure that the scheduling probabilities for each device are recomputed to reflect the exclusion of previously scheduled devices.
The proposed method provides closed-form solutions of scheduling probabilities and avoids high-complexity searching.

\begin{remark}\label{remark2}
    We note that the proposed PO-FL framework can be adapted to the existing probabilistic device scheduling solutions, including: 1) importance-aware scheduling \cite{zhang2022communication}: devices with larger gradient importance $\frac{m_i}{M} \|\bm{g}_i^t\|_2$ are preferred, which gives the single-device scheduling probabilities by $p^t_i = \frac{ \frac{m_i}{M} \| \bm{g}^{t}_{i} \|_2 }{\sum_{j\in \mathcal{N}} \frac{m_j}{M} \| \bm{g}^{t}_{j} \|_2}, \forall i \in \mathcal{N}$; 2) channel-aware scheduling \cite{ma2021user,amiri2021convergence}: devices with higher channel quality $|h^{t}_{i}|^2$ are preferred, which gives the single-device scheduling probabilities by $p^t_i = \frac{ |h^{t}_{i}|^2 }{\sum_{j\in \mathcal{N}} |h^{t}_{j}|^2}, \forall i \in \mathcal{N}$.
    However, we emphasize that our proposed approach in \eqref{eq:single} outperforms these two strategies. On one hand, the importance-aware scheduling strategy overlooks the communication quality and may lead to large aggregation distortion. On the other hand, the channel-aware device scheduling strategy assigns higher probabilities to devices with better channel conditions, which weakens the transmit signal $\frac{m_i a^t}{M p_i^t h_{i}^{t}} \bm{s}_i^t$ for these devices and makes the model aggregation more vulnerable to the channel noise. In contrast, our proposed scheduling approach considers both the gradient importance and communication quality to embrace their respective advantages. Moreover, we assign smaller probabilities to the devices with better channel gains to amplify their transmit signals to be more resistant to channel noise. By doing so, we can maintain the training performance improvement even under large channel noise, as verified by simulation results in the next section.
\end{remark}

\section{Simulation Results}\label{sec:simulation}

\subsection{Simulation Setup}\label{sec:setup}

We simulate an edge network consisting of one server and $N=30$ devices.
The devices are uniformly distributed around the server at a distance of $d_i$ between $10$ to $50$ meters.
The wireless channels from the devices to the server follow the IID Rayleigh fading model, i.e., $h_i^t = \sqrt{g_i} \lambda_i^t, \forall i \in \mathcal{N},$ where $g_i$ denotes the path loss and $\lambda_i^t \sim \mathcal{CN}(0,1)$.
The path loss follows the free-space model $g_i = G_0\left( \frac{3\times 10^{8}}{4\pi f_0 d_i} \right)^{PL}$, where $G_0=4.11$, $f_0=915$ MHz, and $PL=3.76$ denote the antenna gain, the carrier frequency, and the path loss exponent, respectively\cite{lin2022relay}.
We set the transmit power budget $P=1$ W and the noise power $\sigma_z^2 = 10^{-11}$ W unless otherwise stated.

We consider the image classification task on the MNIST \cite{mnist} and CIFAR-10 \cite{cifar10} datasets.
The MNIST dataset contains ten classes of $M=60,000$ training data samples, while the CIFAR-10 dataset contains ten classes of $M=50,000$ training data samples.
We assign each device two random classes of $\left\lfloor\frac{M}{N}\right\rfloor$ data samples. 
Specifically, we sort $M$ training samples by their labels, divide them into $60$ shards of size $\left\lfloor\frac{M}{2N}\right\rfloor$, and assign each device two random shards \cite{li2021silo,sun2021semi}.
We train a logistic regression model on the MNIST dataset to evaluate the effectiveness of our proposed device scheduling method in the convex setting, and train a CNN with four convolutional layers\footnote{This CNN architecture is adapted from \url{https://github.com/dmholtz/cnn-cifar10-pytorch} and achieves an accuracy of 66.6\% on the CIFAR-10 dataset through centralized training.} on the CIFAR-10 dataset in the non-convex case.
The batch size is set as $10$, and we adopt a decaying learning rate $\eta^t = \max \{\eta^0 \times 0.95^{t}, 10^{-5} \} $, where the initial learning rate $\eta^0 = 0.1$ for the MNIST dataset and $\eta^0 = 0.5$ for the CIFAR-10 dataset.
Unless otherwise specified, we set $\alpha = 0.1$ and $|\mathcal{S}^t| = 10, \forall t\in [T] $.
The numerical results are averaged over $10$ independent trials.
We compare the proposed device scheduling policy with the following baselines:
\begin{itemize}
    \item \textbf{Deterministic scheduling}: The server randomly selects a subset of devices $\mathcal{S}^t$ and directly aggregates their gradients as $\sum_{i\in \mathcal{S}^t} \frac{m_i}{\sum_{j\in \mathcal{S}^t} m_j} \bm{g}^{t}_{i}$. This approach does not take into account the gradient importance or communication channel quality of individual devices, and instead relies solely on random selection.
    \item \textbf{Importance-aware scheduling}: Under the PO-FL framework, this approach prioritizes devices with larger values of the local gradient norms. Specifically, the scheduling probability for device $i$ is $p^t_i = \frac{ \frac{m_i}{M} \| \bm{g}^{t}_{i} \|_2 }{\sum_{j\in \mathcal{N}} \frac{m_j}{M} \| \bm{g}^{t}_{j} \|_2}$.
    \item \textbf{Channel-aware scheduling}: In this approach, devices with better channel conditions are given higher probabilities of being scheduled under the PO-FL framework. Specifically, the scheduling probability for device $i$ is $p^t_i = \frac{ |h^{t}_{i}|^2 }{\sum_{j\in \mathcal{N}} |h^{t}_{j}|^2}$.
    \item \textbf{Noise-free scheduling}: This method operates under the PO-FL framework and assumes an idealized scenario without channel noise. In this case, the scheduling probabilities are obtained by solving Problem (P2) with $\sigma_z^2=0$. Although this cannot be a realistic scenario, it provides a performance benchmark for other methods.
\end{itemize}

\subsection{Results}

\subsubsection{Single-Device Scheduling}
To verify the effectiveness of the proposed metrics in \eqref{eq:single}, we present the test accuracy of various device scheduling methods when only one device is scheduled in each communication round in Fig. \ref{fig:single}.
We observe that the deterministic scheduling method results in poor training performance, particularly when tackling the more challenging CIFAR-10 dataset.
In contrast, the proposed scheduling policy achieves the fastest convergence and performs closely to the idealized noise-free case.
Notably, the importance-aware scheduling method exhibits degraded learning performance due to communication distortion incurred by AirComp.
Interestingly, the PO-FL framework fails to converge when using the channel-aware scheduling method.
This is because assigning lower probabilities to devices with worse channel conditions increases the aggregation weights $\left\{\frac{m_i}{M p_i^t}\right\}$, leading to larger communication distortion.
In contrast, the proposed channel and gradient-importance aware device scheduling method assigns larger aggregation weights to devices with better channel conditions, thereby alleviating the negative effects of wireless fading and channel noise.

\begin{figure}[t]
  \centering
  \begin{subfigure}[b]{0.8\columnwidth}
    \centering
    \includegraphics[width=\textwidth]{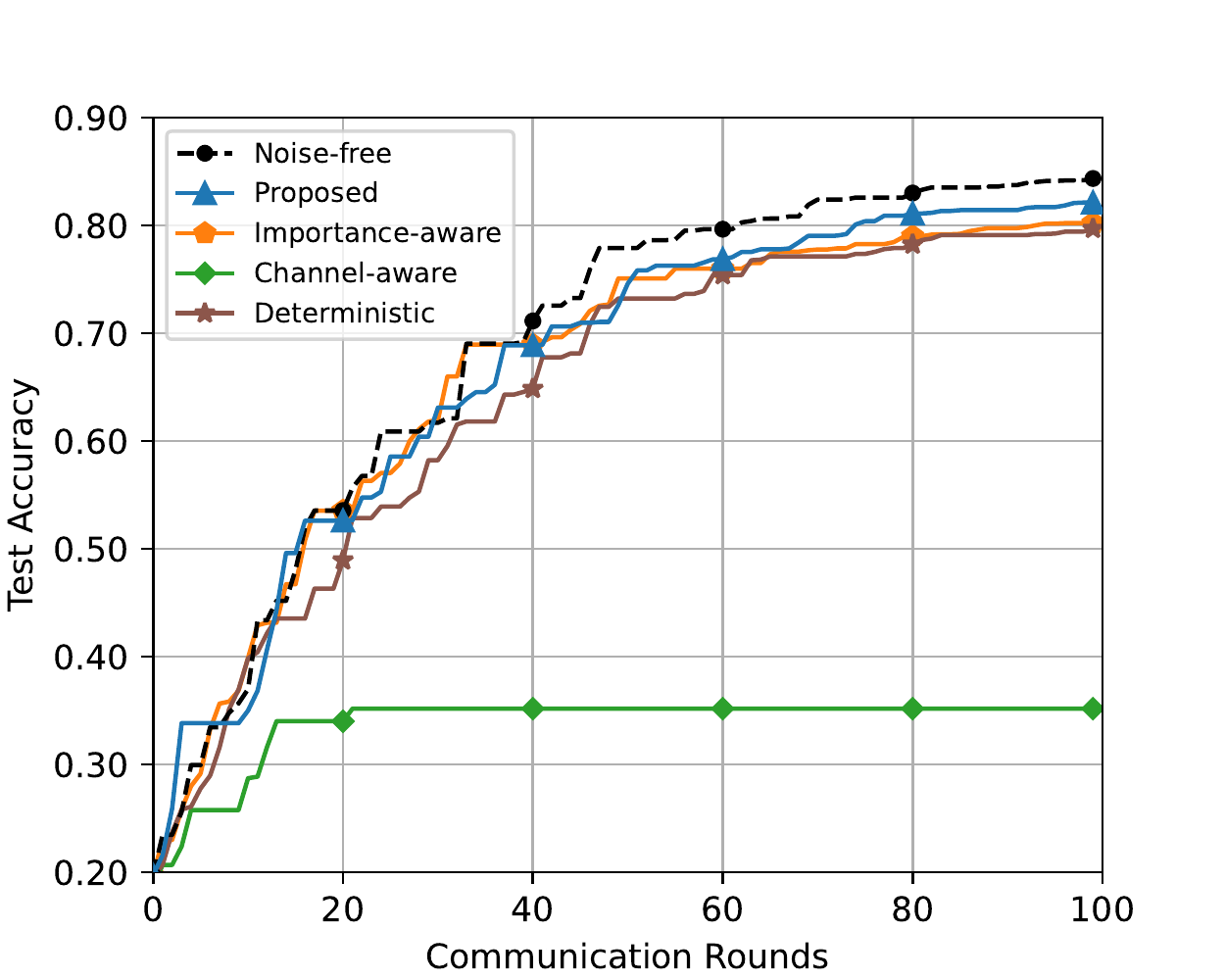}
    \caption{MNIST}
    \label{fig:3-1}
  \end{subfigure}
  \begin{subfigure}[b]{0.8\columnwidth}
    \centering
    \includegraphics[width=\textwidth]{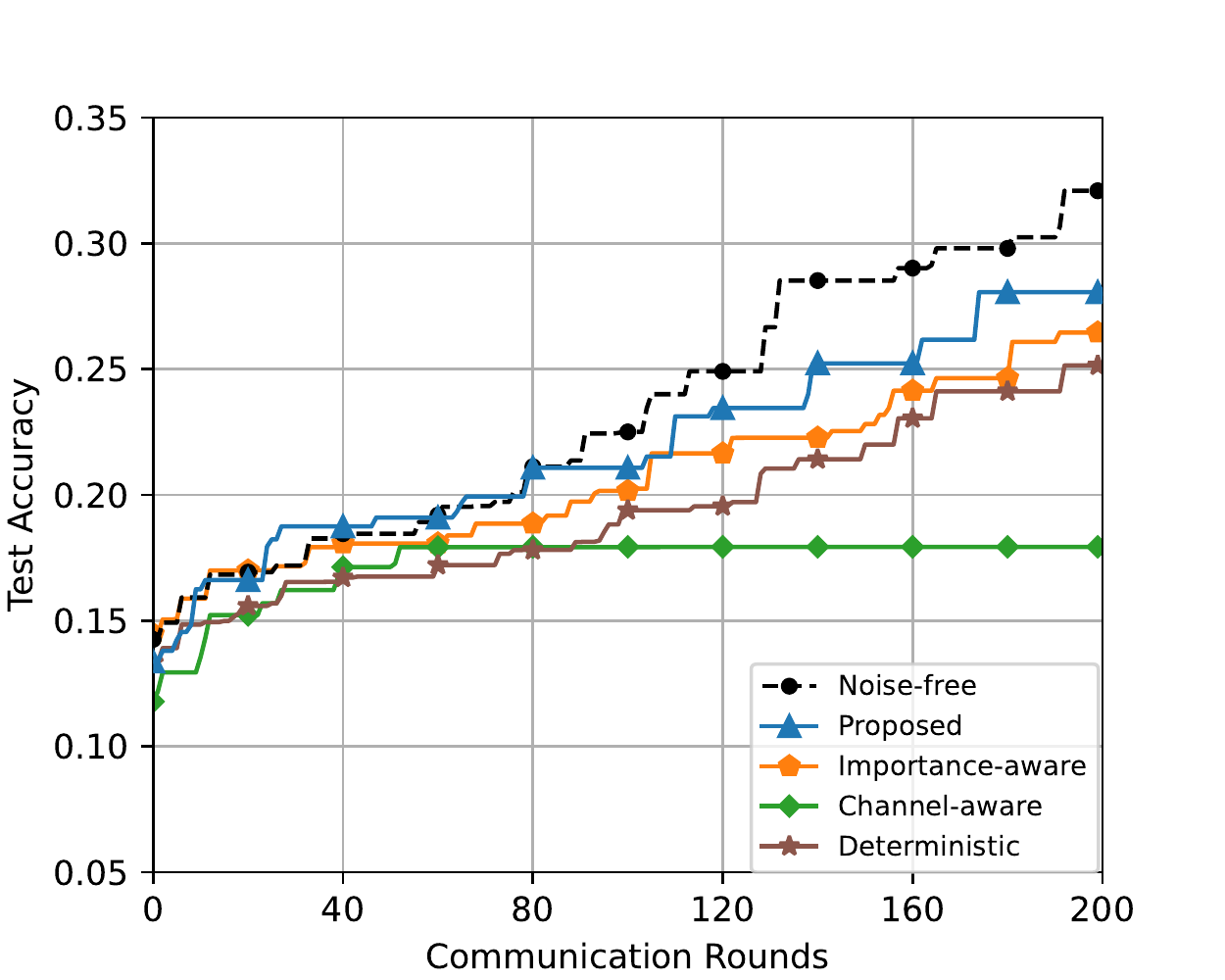}
    \caption{CIFAR-10}
    \label{fig:3-2}
  \end{subfigure}
  \caption{Test accuracy of different device scheduling methods ($|\mathcal{S}^t|=1$) vs. communication rounds on (a) the MNIST dataset and (b) the CIFAR-10 dataset.}
  \label{fig:single}
\end{figure}

\subsubsection{Multi-Device Scheduling}

Next, we evaluate the performance of scheduling $|\mathcal{S}^t|=10$ devices in each communication round and present the results in Fig. \ref{fig:multi}. Compared with the single-device case, all scheduling methods show an improvement in test accuracy.
This is because scheduling more devices allows the server to collect a larger number of gradients, thereby accelerating the FL process.
Similar to the single-device case, our proposed scheduling method achieves the fastest convergence. Noticeably, it also attains the same test accuracy as the idealized noise-free case. This is because our optimization of the device scheduling balances the devices with different channel conditions, which avoids significant communication distortion while assuring the importance of received gradients.
Moreover, it is worth noting that the deterministic device scheduling method directly aggregates the received gradients without reweighting, leading to its slower convergence than most probabilistic device scheduling methods.
However, it avoids improper device selection and aggregation weight design in the channel-aware device scheduling method and thus achieves better learned model accuracy.
Since the deterministic scheduling method cannot achieve unbiased aggregation, we exclude it from the following studies and focus on different scheduling policies under the PO-FL framework.
\begin{figure}[!t]
  \centering
  \begin{subfigure}[b]{0.8\columnwidth}
    \centering
    \includegraphics[width=\textwidth]{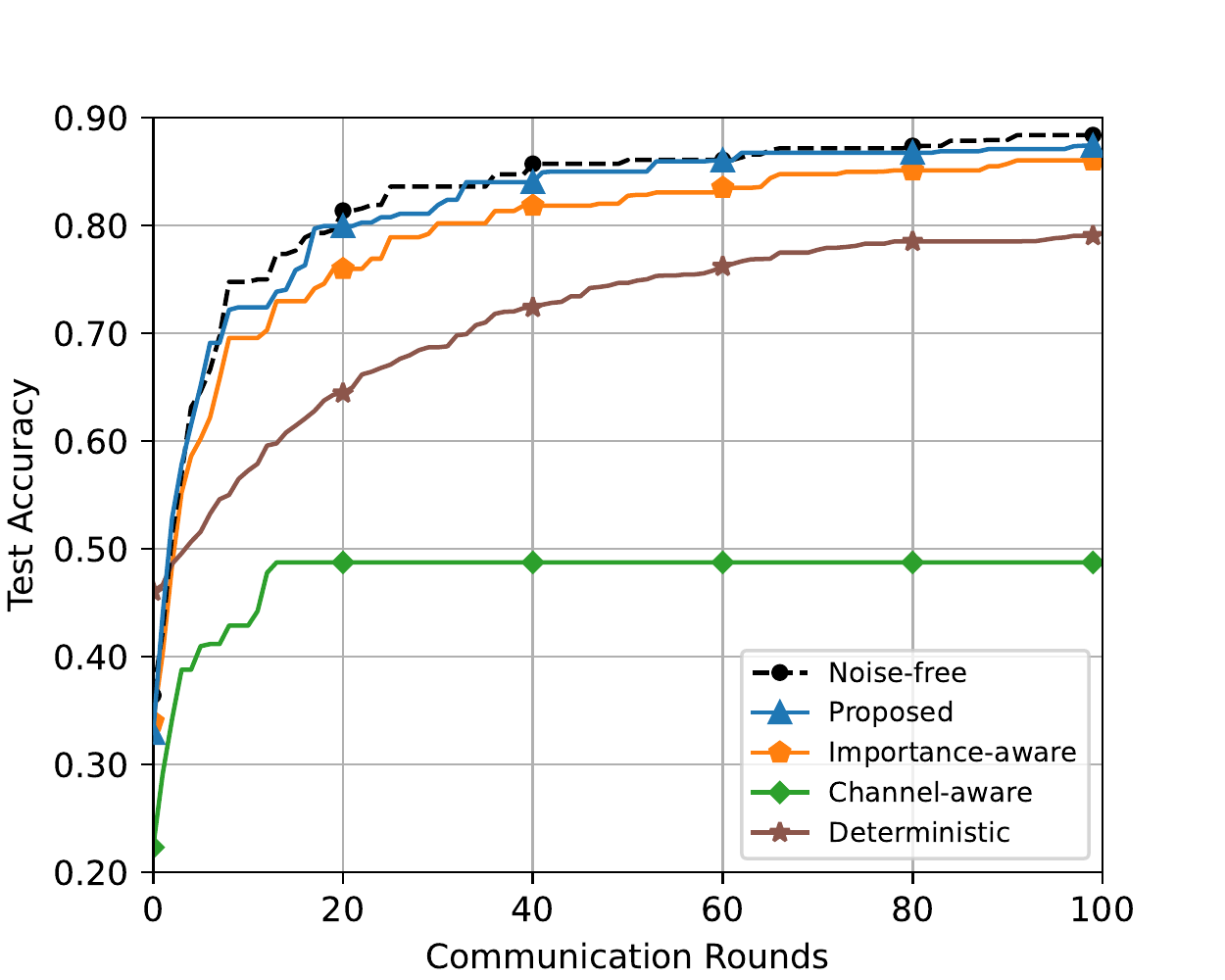}
    \caption{MNIST}
    \label{fig:4-1}
  \end{subfigure}
  \begin{subfigure}[b]{0.8\columnwidth}
    \centering
    \includegraphics[width=\textwidth]{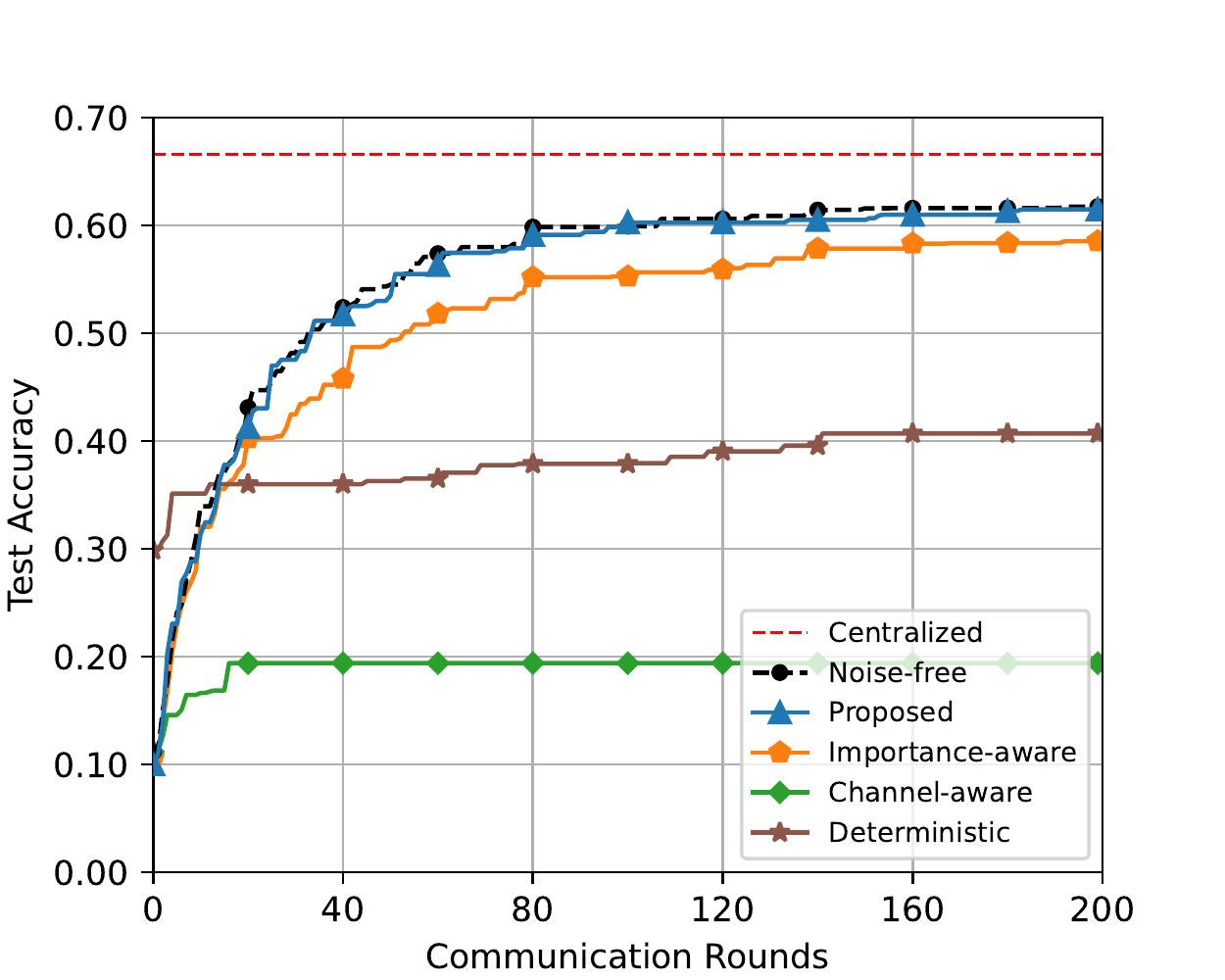}
    \caption{CIFAR-10}
    \label{fig:4-2}
  \end{subfigure}
  \caption{Test accuracy of different device scheduling methods vs. communication rounds on (a) the MNIST dataset and (b) the CIFAR-10 dataset.}
  \label{fig:multi}
\end{figure}

\subsubsection{Effect of the Channel Noise}

We vary the noise power $\sigma_z^2$ and compare the test accuracy of the learned model under the PO-FL framework in Fig. \ref{fig:snr}.
We report the best test accuracy over 100 (respectively 200) communication rounds on the MNIST (respectively CIFAR-10) dataset.
We observe that as the noise power increases, the model performance with different scheduling policies degrades due to more significant communication distortion.
Moreover, the proposed scheduling method outperforms the baseline methods by admitting devices with more important gradient information and better channel quality.
The performance improvement is particularly significant in the noise-limited regime, i.e., $\sigma_z^2=10^{-10} \sim 10^{-9}$, where the negative effect of communication distortion is more prominent.
In addition, the channel-aware device scheduling method leads to notably poor model accuracy, as it aggravates the communication distortion by increasing the aggregation weights for those devices with weak channel conditions, as explained in Remark \ref{remark2}.

\begin{figure}[!t]
  \centering
  \begin{subfigure}[b]{0.8\columnwidth}
    \centering
    \includegraphics[width=\textwidth]{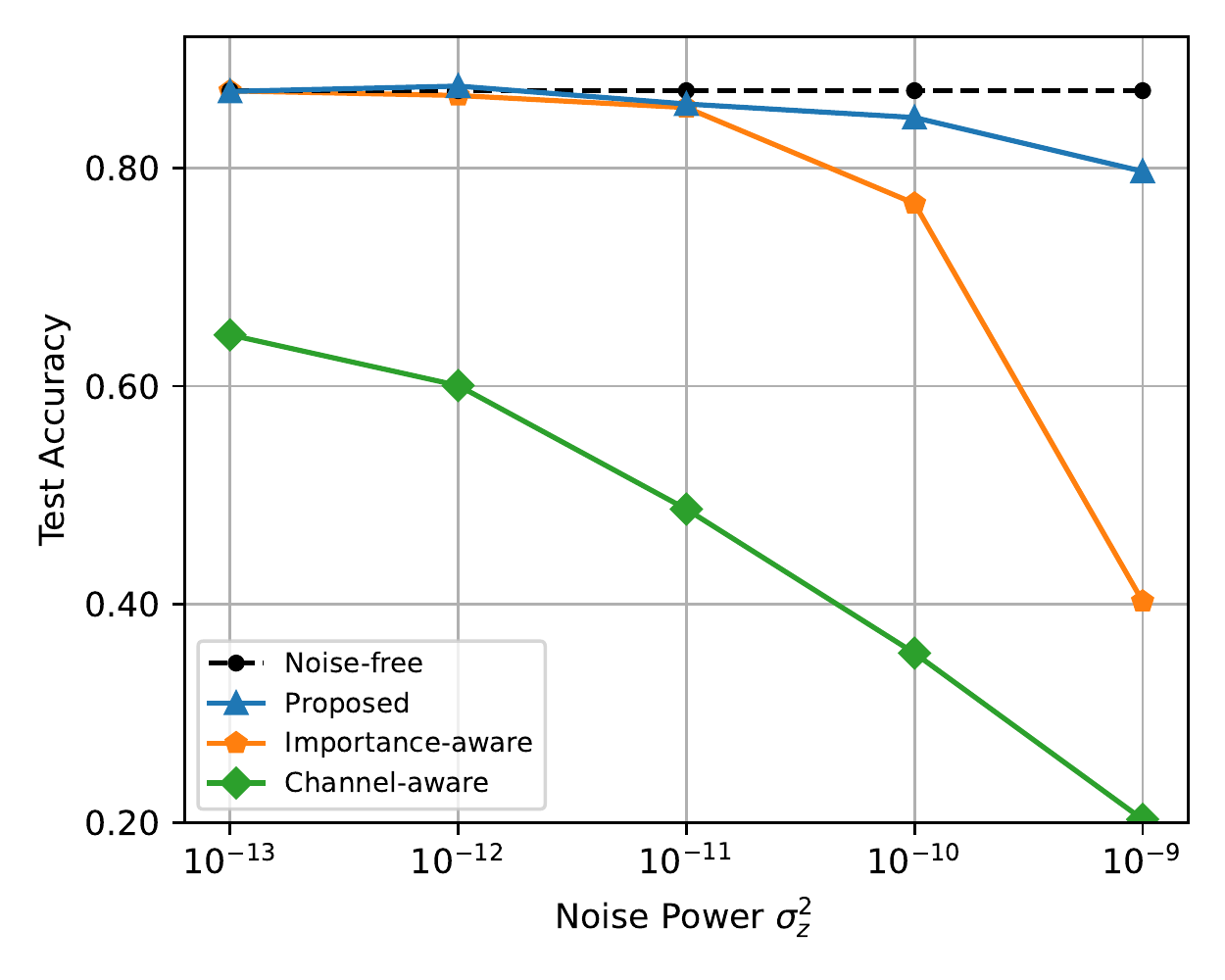}
    \caption{MNIST}
    \label{fig:5-1}
  \end{subfigure}
  \begin{subfigure}[b]{0.8\columnwidth}
    \centering
    \includegraphics[width=\textwidth]{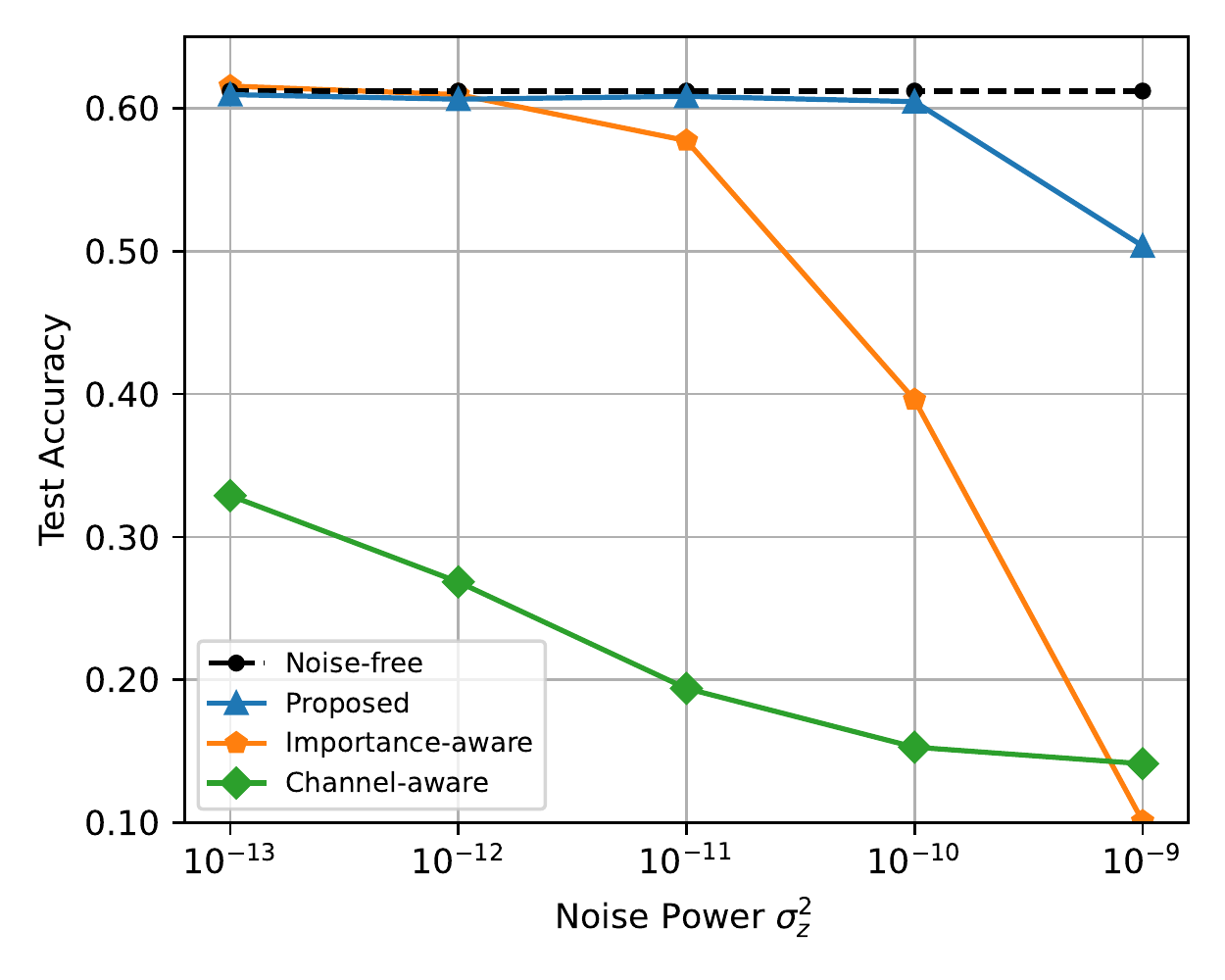}
    \caption{CIFAR-10}
    \label{fig:5-2}
  \end{subfigure}
  \caption{Test accuracy with different noise power $\sigma_z^2$ after (a) 100 communication rounds on the MNIST dataset and (b) 200 communication rounds on the CIFAR-10 dataset.}
  \label{fig:snr}
\end{figure}

\subsubsection{Effect of the Number of Scheduled Devices}
We present the test accuracy of the learned models with varying numbers of scheduled devices $|\mathcal{S}^t|$ in each communication round in Fig. \ref{fig:device}.
We exclude the channel-aware scheduling method from this comparison since it consistently exhibits poor performance.
As the number of scheduled devices $|\mathcal{S}^t|$ increases from $1$ to $20$, all the methods yield better models since more gradients can be collected by the server to accelerate the FL process.
However, when it further increases to $|\mathcal{S}^t|=30$, the learned model accuracy is degraded due to the increased communication distortion.
This observation illustrates the tradeoff between the global update variance and communication distortion due to the number of scheduled devices $|\mathcal{S}^t|$. Therefore, it is important to find a suitable value of $|\mathcal{S}^t|$ in the proposed algorithm.
Moreover, thanks to the optimized scheduling probabilities, the proposed method consistently achieves faster convergence than the baselines.
The performance improvement is especially noticeable when fewer devices are selected, i.e., $|\mathcal{S}^t|=1$ or $5$, since the selected devices differ significantly among different methods.

\begin{figure}[!t]
  \centering
  \begin{subfigure}[b]{0.8\columnwidth}
    \centering
    \includegraphics[width=\textwidth]{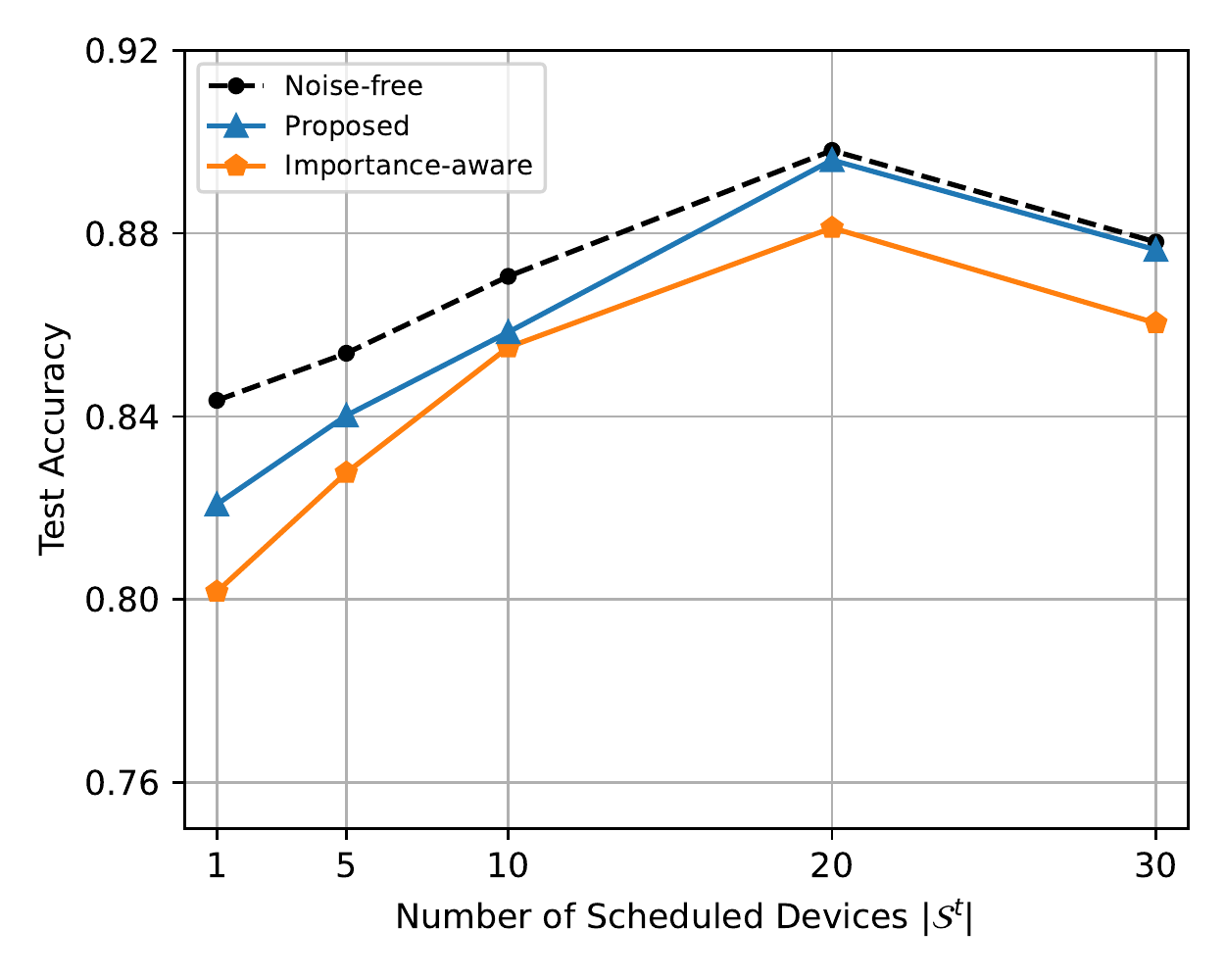}
    \caption{MNIST}
    \label{fig:6-1}
  \end{subfigure}
  \begin{subfigure}[b]{0.8\columnwidth}
    \centering
    \includegraphics[width=\textwidth]{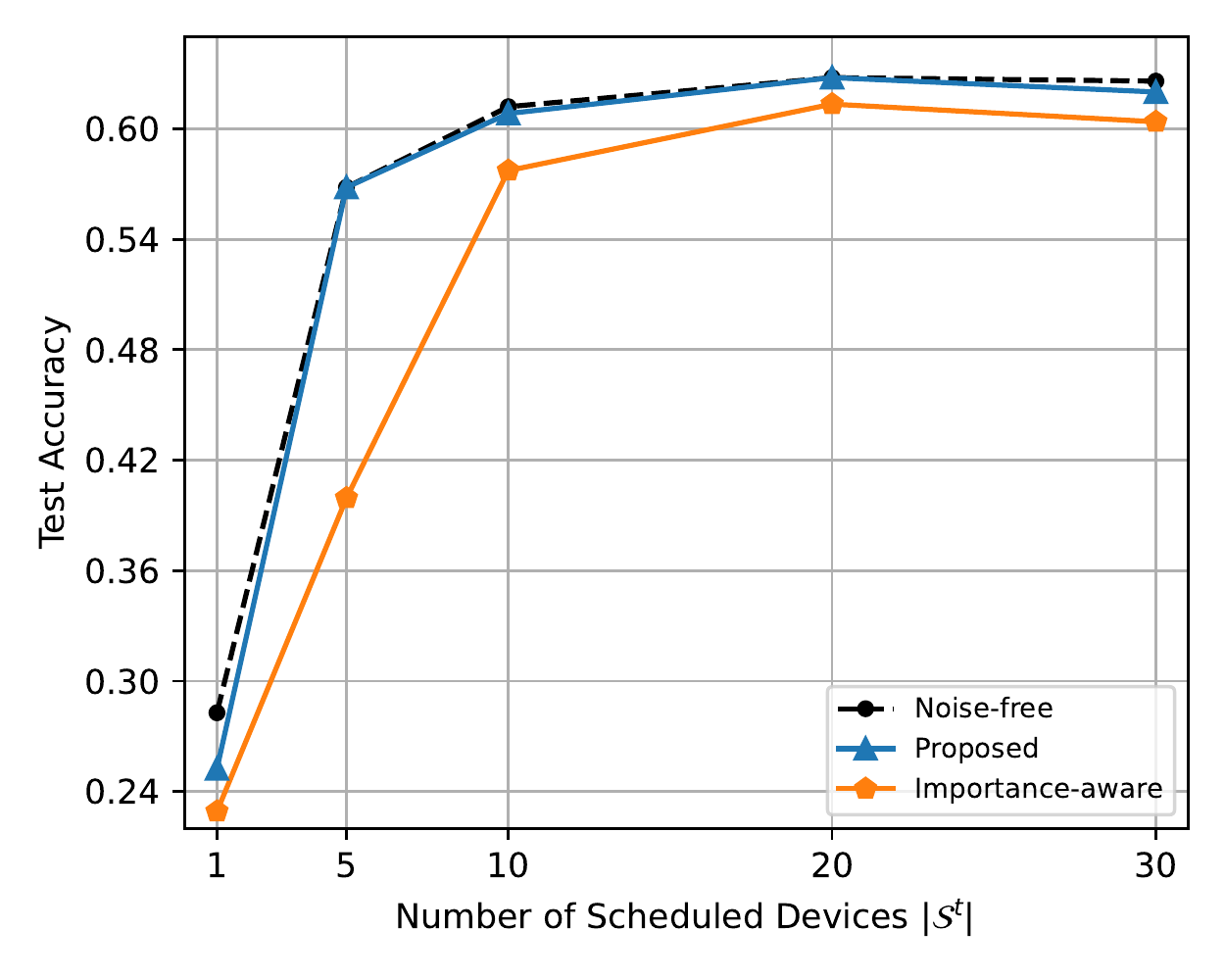}
    \caption{CIFAR-10}
    \label{fig:6-2}
  \end{subfigure}
  \caption{Test accuracy with different numbers of devices vs. communication rounds on (a) the MNIST dataset and (b) the CIFAR-10 dataset.}
  \label{fig:device}
\end{figure}

\subsubsection{Effect of the Data Heterogeneity}

\begin{figure}[!t]
  \centering
  \begin{subfigure}[b]{0.8\columnwidth}
    \centering
    \includegraphics[width=\textwidth]{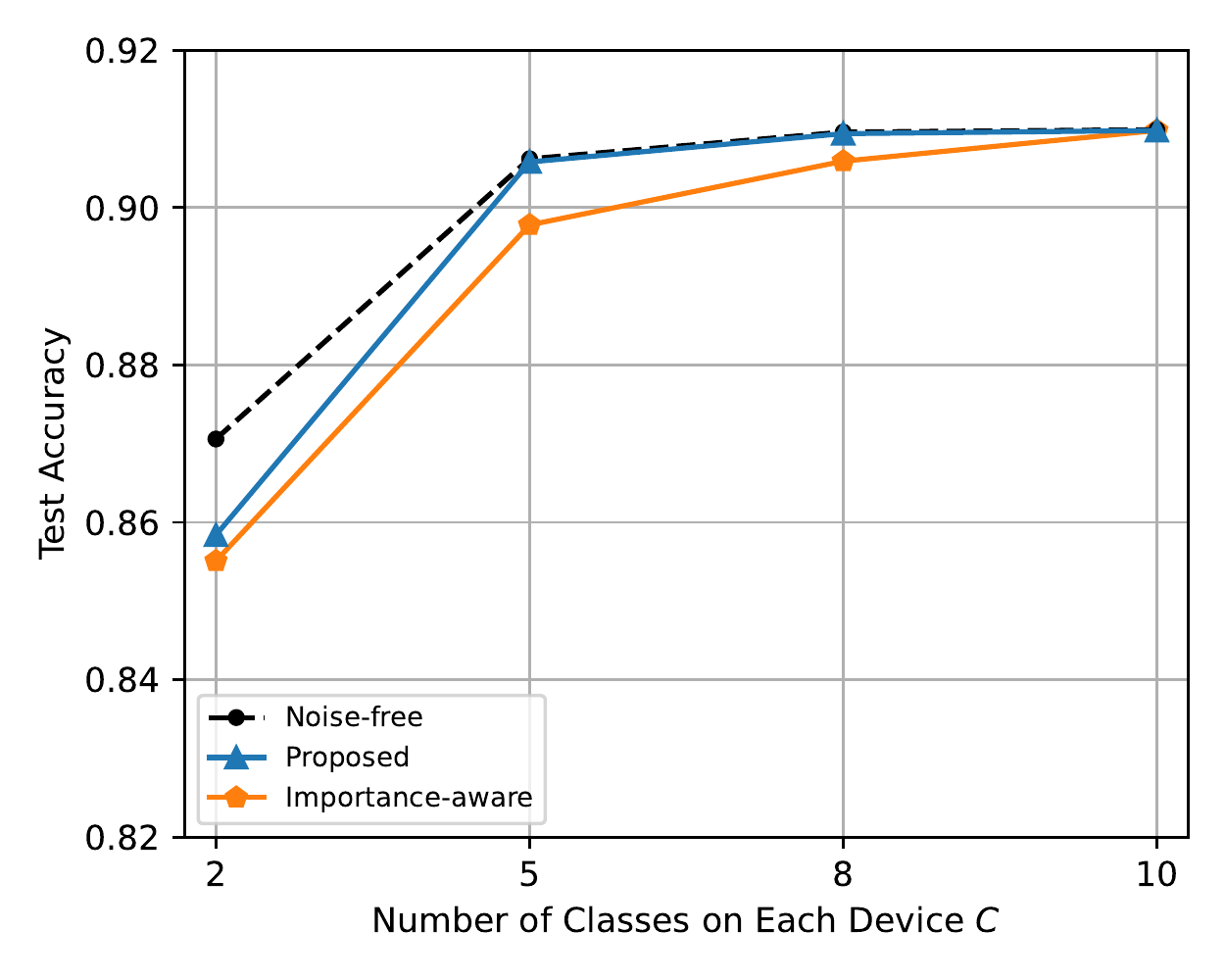}
    \caption{MNIST}
    \label{fig:7-1}
  \end{subfigure}
  \begin{subfigure}[b]{0.8\columnwidth}
    \centering
    \includegraphics[width=\textwidth]{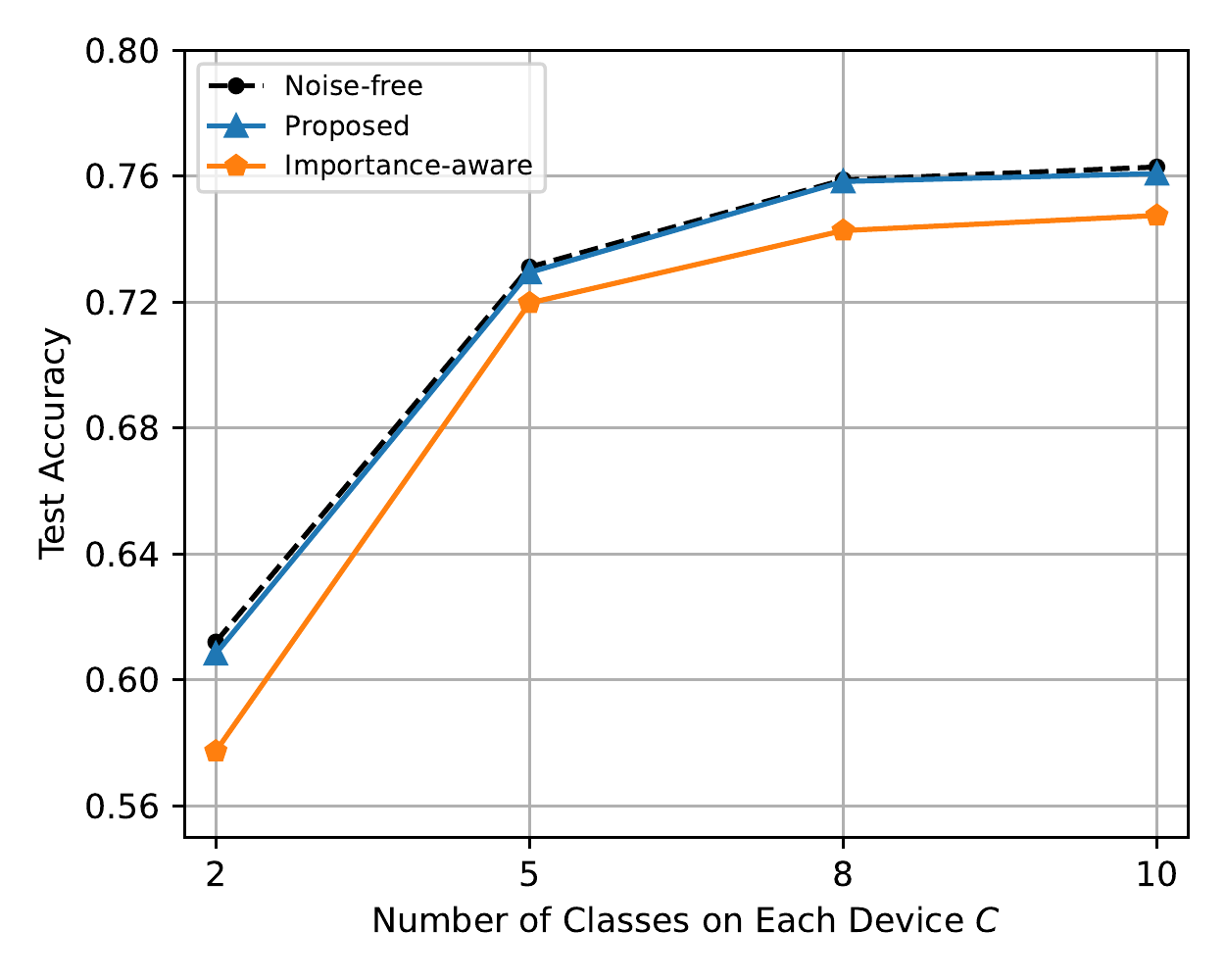}
    \caption{CIFAR-10}
    \label{fig:7-2}
  \end{subfigure}
  \caption{Test accuracy with different data heterogeneity on (a) the MNIST dataset and (b) the CIFAR-10 dataset.}
  \label{fig:non-iid}
\end{figure}

In Fig. \ref{fig:non-iid}, we study the impact of data heterogeneity on the training performance.
Specifically, we vary the data heterogeneity by changing the number of classes of data at each device, denoted by $C$, with an equal number of training samples of $\frac{M}{NC}$ for each class.
The results show that a higher degree of data heterogeneity, i.e., a smaller value of $C$, leads to a slower training speed due to the divergence of gradients.
In this case, the proposed device scheduling policy selects more important gradients while avoiding significant communication distortion, leading to a noticeable improvement in test accuracy.
Moreover, when the data distribution is close to IID, e.g., $C=8$ and $10$, the proposed design approaches the performance upper bound achieved in the noise-free case.

\subsubsection{Effect of the Value of $\alpha$}
In Fig. \ref{fig:alpha} and Table \ref{table:alpha}, we evaluate the test accuracy of the proposed scheduling method with different values of $\alpha$ under varying communication conditions. We observe that as the communication condition deteriorates (i.e., the noise power $\sigma_z^2$ increases), the optimal $\alpha$ tends to take on larger values. This trend can be attributed to the fact that as communication distortion becomes more pronounced, the overall model performance is more severely deteriorated. To mitigate the negative impact and ensure effective model training, it becomes necessary to place higher weight on reducing communication distortion. Thus, a larger value of $\alpha$ is more suitable under such challenging communication conditions, as it emphasizes minimizing distortion during device scheduling.
For example, consider a scenario where unmanned aerial vehicles (UAVs) acting as edge devices fly in urban environments with tall buildings or natural obstacles, the communication link quality may be severely affected. In this case, it is more favorable to use a large value of $\alpha$ to avoid severe communication distortion.
On the other hand, for the Internet of Things (IoT) devices that are distributed close to the server and usually have good communication conditions, a small value of $\alpha$ should be set to prioritize important local updates.

\begin{figure}[!h]
\setlength\abovedisplayskip{2em}
    \centering
    \includegraphics[width=0.8\columnwidth]{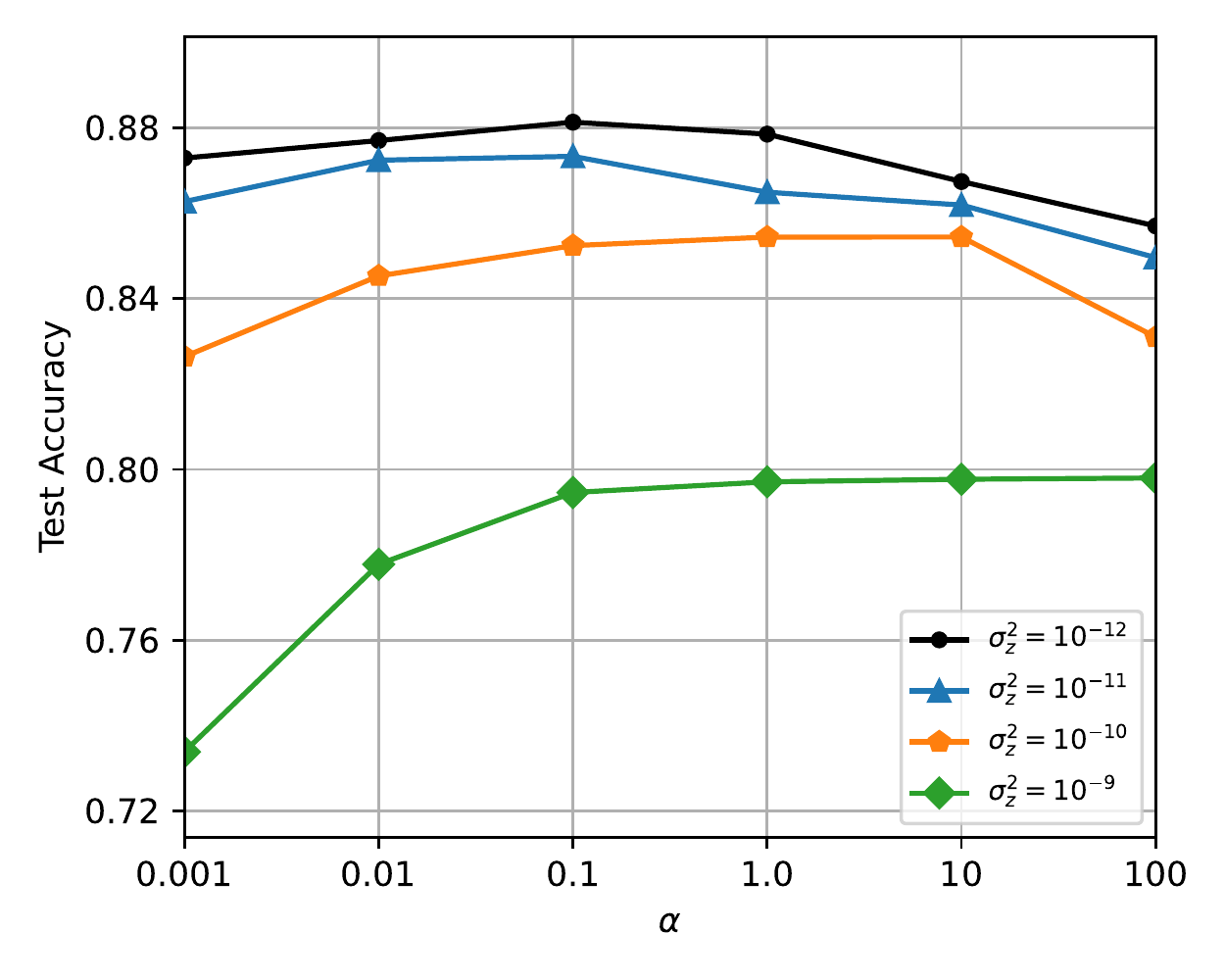}
    \caption{Test accuracy with different values of $\alpha$ after 100 communication rounds on the MNIST dataset.}
    \label{fig:alpha}
\end{figure}

\begin{table}[!h] 
\centering 
\captionof{table}{Test accuracy with different values of $\alpha$ after 100 communication rounds on the MNIST dataset. The best accuracy is highlighted in \textbf{bold}.}
\label{table:alpha}
 \resizebox{\linewidth}{!}{
\begin{tabular}{|c|c|c|c|c|c|c|} 
\hline Test Accuracy & $\alpha=0.001$ & $\alpha=0.01$ & $\alpha=0.1$ & $\alpha=1$ & $\alpha=10$ & $\alpha=100$ \\ \hline
$\sigma_z^2=10^{-9}$ & 0.7339 & 0.7778 & 0.7946 & 0.7971 & 0.7977 & \textbf{0.7980}  \\ \hline 
$\sigma_z^2=10^{-10}$ & 0.8264 & 0.8453 & 0.8524 & 0.8544 & \textbf{0.8544} & 0.8310 \\ \hline 
$\sigma_z^2=10^{-11}$ & 0.8627 & 0.8724 & \textbf{0.8733} & 0.8649 & 0.8619 & 0.8496 \\ \hline 
$\sigma_z^2=10^{-12}$ & 0.8729 & 0.8770 & \textbf{0.8813} & 0.8785 & 0.8674 & 0.857 \\ \hline 
\end{tabular}
}
\end{table}

\section{Conclusions and Future Works}\label{sec:conclusion}
In this paper, we proposed a novel framework named PO-FL that leverages probabilistic device scheduling to enhance the learning performance of over-the-air FL.
We analyzed the convergence of PO-FL for both convex and non-convex loss functions. Our analysis revealed that the device scheduling policy critically affects the learning performance through the communication distortion and the global update variance.
Based on the analytical results, we formulated an optimization problem of device scheduling to jointly minimize these two factors, which was efficiently solved by an algorithm that considers both channel condition and gradient importance.
Extensive simulation results demonstrated that the proposed design outperforms baseline policies and consistently achieves faster convergence.

For future works, it is worth investigating how to optimize the number of scheduled devices in each communication round to further improve the training performance of the PO-FL framework.
Besides, it is also interesting to study the impact of the number of local training steps on the communication efficiency in the PO-FL framework.
Moreover, considering the benefits of AI/ML-enabled methods in wireless networks \cite{li2021toward,mowla2019federated,DuJWRD20}, it is very promising to explore these methods in the device scheduling designs for over-the-air FL.

\appendices

\section{Proof of Theorem \ref{thm:non-convex}}\label{appendix-non-convex}

To prove Theorem \ref{thm:non-convex}, we first analyze the loss decay in each communication round as follows:
\begin{align}
    & \mathbb{E}[f(\bm{w}^{t+1})] - \mathbb{E}[f(\bm{w}^{t})] \nonumber \\
    \overset{(\text{a})}{\leq} & \mathbb{E}\left[\left\langle \nabla f(\bm{w}^{t}), \bm{w}^{t+1} - \bm{w}^{t}\right\rangle\right] + \frac{L}{2} \mathbb{E}[\| \bm{w}^{t+1} - \bm{w}^{t} \|_2^2] \nonumber \\
    = & - \eta^t \mathbb{E} \left[ \left\langle \nabla f(\bm{w}^{t}), \hat{\bm{y}}^t \right\rangle \right] + \frac{L (\eta^t)^2}{2} \mathbb{E}\left[\left\| \hat{\bm{y}}^t \right\|_2^2 \right] \nonumber \\
    = & - \eta^t \mathbb{E} \left[ \left\langle \nabla f(\bm{w}^{t}), \mathbb{E} [\hat{\bm{y}}^t] \right\rangle \right] \nonumber \\
    & + \frac{L (\eta^t)^2}{2} \mathbb{E}\left[\left\| \hat{\bm{y}}^t - \nabla f(\bm{w}^t) + \nabla f(\bm{w}^t) \right\|_2^2 \right] \nonumber \\
    \overset{(\text{b})}{=} & - \eta^t \mathbb{E} \left[ \left\| \nabla f(\bm{w}^t) \right\|_2^2 \right] 
    + \frac{L (\eta^t)^2}{2} \mathbb{E}\left[\left\| \hat{\bm{y}}^t - \nabla f(\bm{w}^t) \right\|_2^2 \right]  \nonumber \\
    & + \frac{L (\eta^t)^2}{2} \mathbb{E} \left[ \left\| \nabla f(\bm{w}^t) \right\|_2^2 \right],
\label{eq:help-53}
\end{align}
where (a) follows from the $L$-smoothness of the loss function in Assumption \ref{smooth}, and (b) follows from the definition of $\hat{\bm{y}}^t$ and its unbiasedness in Lemma \ref{unbiased}, i.e., $\mathbb{E}[\hat{\bm{y}}^t] = \nabla f(\bm{w}^t)$.

Next, we provide an upper bound for the term $\mathbb{E}\left[\left\| \hat{\bm{y}}^t - \nabla f(\bm{w}^t) \right\|_2^2 \right]$ as follows:
\begin{align}
    & \mathbb{E}\left[\left\| \hat{\bm{y}}^t - \nabla f(\bm{w}^t) \right\|_2^2 \right] \nonumber \\ 
    = & \mathbb{E} \left[\left\| \sum_{i\in \mathcal{S}^t} \frac{m_i}{M p_i^t} \bm{g}^{t}_{i} + \frac{ \sqrt{V_{\bm{g}}^{t}}}{a^{t}} \bm{z}^t - \sum_{i\in \mathcal{N}} \frac{m_i}{M} \nabla f_i(\bm{w}^t) \right\|_2^2 \right] \nonumber \\
    \overset{(\text{c})}{\leq} & (1+\alpha) \underbrace{\mathbb{E}\left[\left\| \frac{ \sqrt{V_{\bm{g}}^{t}}}{a^{t}} \bm{z}^t \right\|_2^2 \right]}_{\mathbb{E}[e_{\text{com}}^{t}]}  \nonumber \\
    & + \left( 1+\frac{1}{\alpha} \right) \mathbb{E} \left[\left\| \sum_{i\in \mathcal{S}^t} \frac{m_i}{M p_i^t} \bm{g}^{t}_{i} - \sum_{i\in \mathcal{N}} \frac{m_i}{M} \nabla f_i(\bm{w}^t) \right\|_2^2 \right],
\label{eq:help-57}
\end{align}
where (c) follows from the inequality $\left\|\bm{x} + \bm{y} \right\|_2^2 \leq (1+\alpha) \left\|\bm{x} \right\|_2^2 + \left( 1+\frac{1}{\alpha} \right) \left\|\bm{y} \right\|_2^2$, $\forall \bm{x}, \bm{y}$, and $\alpha>0$. 

Moreover, we decompose the second term in the RHS of \eqref{eq:help-57} as follows:
\begin{align}
    & \mathbb{E} \left[\left\| \sum_{i\in \mathcal{S}^t} \frac{m_i}{M p_i^t} \bm{g}^{t}_{i} - \sum_{i\in \mathcal{N}} \frac{m_i}{M} \nabla f_i(\bm{w}^t) \right\|_2^2 \right]  \nonumber \\
    =&\mathbb{E} \left[\left\| \sum_{i\in \mathcal{S}^t} \frac{m_i}{M p_i^t} \bm{g}^{t}_{i} \!-\! \sum_{i\in \mathcal{N}} \frac{m_i}{M} \bm{g}^{t}_{i}
    \!+\! \sum_{i\in \mathcal{N}} \frac{m_i}{M} \bm{g}^{t}_{i} \!-\! \sum_{i\in \mathcal{N}} \frac{m_i}{M} \nabla f_i(\bm{w}^t) \right\|_2^2 \right] \nonumber\\
    = & \underbrace{\mathbb{E} \left[\left\| \sum_{i\in \mathcal{S}^t} \frac{m_i}{M p_i^t} \bm{g}^{t}_{i} - \sum_{i\in \mathcal{N}} \frac{m_i}{M} \bm{g}^{t}_{i} \right\|_2^2 \right]}_{\mathbb{E}[e_{\text{var}}^{t}]} \nonumber \\
    & + \mathbb{E} \left[\left\| \sum_{i\in \mathcal{N}} \frac{m_i}{M} \bm{g}^{t}_{i} - \sum_{i\in \mathcal{N}} \frac{m_i}{M} \nabla f_i(\bm{w}^t) \right\|_2^2 \right] \nonumber \\
    & + \mathbb{E}\left[ \left\langle \sum_{i\in \mathcal{S}^t}\! \frac{m_i}{M p_i^t} \bm{g}^{t}_{i} \!-\! \sum_{i\in \mathcal{N}}\! \frac{m_i}{M} \bm{g}^{t}_{i} ,\! \sum_{i\in \mathcal{N}}\! \frac{m_i}{M} \bm{g}^{t}_{i} \!-\! \sum_{i\in \mathcal{N}}\! \frac{m_i}{M} \nabla f_i(\bm{w}^t) \right\rangle \right] \nonumber \\
    \overset{(\text{d})}{=} & \mathbb{E}[e_{\text{var}}^{t}] + \mathbb{E} \left[\left\| \sum_{i\in \mathcal{N}} \frac{m_i}{M} \bm{g}^{t}_{i} - \sum_{i\in \mathcal{N}} \frac{m_i}{M} \nabla f_i(\bm{w}^t) \right\|_2^2 \right] \nonumber \\
    \overset{(\text{e})}{\leq} & \mathbb{E}[e_{\text{var}}^{t}] + \sigma^2 ,
\label{eq:gap}
\end{align}
where (d) is because $\mathbb{E}[\bm{g}^{t}_{i}] = \nabla f_i(\bm{w}^t)$, and (e) follows from Assumption \ref{sgd}.

By substituting \eqref{eq:help-57} and \eqref{eq:gap} in \eqref{eq:help-53}, we have:
\begin{align}
    & \mathbb{E}[f(\bm{w}^{t+1})] - \mathbb{E}[f(\bm{w}^{t})] \nonumber \\
    \leq &- \left(\eta^t - \frac{L (\eta^t)^2}{2} \right) \mathbb{E} \left[ \left\| \nabla f(\bm{w}^t) \right\|_2^2 \right] + \frac{L (\eta^t)^2}{2} \left( 1+\frac{1}{\alpha} \right) \sigma^2 \nonumber \\
    & + \frac{L (\eta^t)^2}{2} \left[ (1+\alpha) \mathbb{E}[e_{\text{com}}^{t}] + \left( 1+\frac{1}{\alpha} \right) \mathbb{E}[e_{\text{var}}^{t}] \right] \nonumber \\
    \overset{\text{(f)}}{\leq} & - \frac{\eta^t}{2} \mathbb{E} \left[ \left\| \nabla f(\bm{w}^t) \right\|_2^2 \right] + \frac{L (\eta^t)^2}{2} \left( 1+\frac{1}{\alpha} \right) \sigma^2 \nonumber \\
    & + \frac{L (\eta^t)^2}{2} \left[ (1+\alpha) \mathbb{E}[e_{\text{com}}^{t}] + \left( 1+\frac{1}{\alpha} \right) \mathbb{E}[e_{\text{var}}^{t}] \right], \label{eq:one-round} 
\end{align}
where (f) is because $\eta^t \leq \frac{1}{L}$.

By summing both sides of \eqref{eq:one-round} up over $t=0,1,\dots,T-1$, we have:
\begin{align}
    & \sum_{t=0}^{T-1} \mathbb{E}[f(\bm{w}^{t+1})] - \mathbb{E}[f(\bm{w}^{t})] \nonumber \\
    \leq & - \sum_{t=0}^{T-1} \frac{\eta^t}{2} \mathbb{E} \left[ \left\| \nabla f(\bm{w}^t) \right\|_2^2 \right] + \sum_{t=0}^{T-1} \frac{ L (\eta^t)^2}{2} \left( 1+\frac{1}{\alpha} \right) \sigma^2 \nonumber \\
    & + \sum_{t=0}^{T-1} \frac{L (\eta^t)^2 }{2} \left[ (1+\alpha) \mathbb{E}[e_{\text{com}}^{t}] + \left( 1+\frac{1}{\alpha} \right) \mathbb{E}[e_{\text{var}}^{t}] \right]. \label{eq:help-62}
\end{align}

By rearranging the terms in \eqref{eq:help-62} and dividing both sides by $\frac{1}{2} \gamma_T = \frac{1}{2} \sum_{t=0}^{T-1} \eta^t$, we obtain the following result:
\begin{align}
    & \frac{1}{\gamma_T} \sum_{t=0}^{T-1} \eta^t \mathbb{E}\left[ \left\| \nabla f(\bm{w}^t) \right\|_2^2 \right] \nonumber \\
    \leq & \frac{2}{\gamma_T} \left( \mathbb{E}[f(\bm{w}^{0})] - \mathbb{E}[f(\bm{w}^{T})] \right) + \frac{L}{\gamma_T} \sum_{t=0}^{T-1} (\eta^t)^2 \left( 1+\frac{1}{\alpha} \right) \sigma^2 \nonumber\\
    & + \frac{L}{\gamma_T} \sum_{t=0}^{T-1} (\eta^t)^2 \left[ (1+\alpha) \mathbb{E}[e_{\text{com}}^{t}] + \left( 1+\frac{1}{\alpha} \right) \mathbb{E}[e_{\text{var}}^{t}] \right] \nonumber \\ 
    \leq & \frac{2}{\gamma_T} \left( \mathbb{E}[f(\bm{w}^{0})] - f(\bm{w}^{*}) \right) + \frac{L}{\gamma_T} \sum_{t=0}^{T-1} (\eta^t)^2 \left( 1+\frac{1}{\alpha} \right) \sigma^2 \nonumber \\
    & + \frac{L}{\gamma_T} \sum_{t=0}^{T-1} (\eta^t)^2 \left[ (1+\alpha) \mathbb{E}[e_{\text{com}}^{t}] + \left( 1+\frac{1}{\alpha} \right) \mathbb{E}[e_{\text{var}}^{t}] \right].
\label{eq:help-63}
\end{align}

Furthermore, we have:
\begin{equation} 
    \min_{t\in[T]} \mathbb{E}\left[ \left\| \nabla f(\bm{w}^t) \right\|_2^2 \right] 
    \leq \frac{1}{\gamma_T} \sum_{t=0}^{T-1} \eta^t \mathbb{E}\left[ \left\| \nabla f(\bm{w}^t) \right\|_2^2 \right].
    \label{eq:help-64}
\end{equation}
Substituting \eqref{eq:help-63} into \eqref{eq:help-64} gives the results in \eqref{eq:non-convex}.
\qed

\vspace{-1em}
\section{Proof of Corollary \ref{corollary2}}\label{proof-corollary2}

To prove Corollary \ref{corollary2}, we first bound the global variance using Assumption \ref{sgd2} as follows:
\begin{align}
    \tilde{V}_{\bm{g}}^{t} = & \sum_{i\in \mathcal{N}} \left(\frac{m_i}{M} \right) V_i^t
    = \sum_{i\in \mathcal{N}} \left(\frac{m_i}{M} \right) \frac{1}{D} \sum_{d=1}^{D} (\bm{g}_{i}^{t}[d] - M_{i}^{t} )^2  \nonumber \\
    \leq & \sum_{i\in \mathcal{N}} \left(\frac{m_i}{M} \right) \frac{1}{D} \sum_{d=1}^{D} (\bm{g}_{i}^{t}[d])^2 \leq \frac{G^2}{D}.
\end{align}

Consequently, the communication distortion is upper bounded as follows:
\begin{align}
    \mathbb{E}[e_{\text{com}}^{t}]
    = & \max_{i\in \mathcal{S}^t} \frac{ (\rho_{i}^{t})^2}{P |h_{i}^{t}|^2} \tilde{V}_{\bm{g}}^{t} D\sigma_z^2 \nonumber \\
    \leq & \max_{i\in \mathcal{S}^t} \frac{ (\rho_{i}^{t})^2}{P |h_{i}^{t}|^2} \sigma_z^2 G^2.
\end{align}

Similarly, we upper bound the global update variance as follows:
\begin{align}
    \mathbb{E}[e_{\text{var}}^{t}] =& \mathbb{E}\left[ \left\| \sum_{i\in \mathcal{S}_t} \frac{m_i}{M p_i^t } \bm{g}^{t}_{i} - \sum_{i\in \mathcal{S}^t} \sum_{j\in \mathcal{N}} \frac{m_j}{M} \bm{g}^{t}_{j} \right\|_2^2 \right] \nonumber \\
    \overset{(\text{a})}{=} & \mathbb{E}\left[ \left\| \sum_{i\in \mathcal{S}_t} \frac{m_i}{M p_i^t } \bm{g}^{t}_{i} \right\|_2^2 \right] - \left[ \left\| \sum_{i\in \mathcal{S}^t} \sum_{j\in \mathcal{N}} \frac{m_j}{M} \bm{g}^{t}_{j} \right\|_2^2 \right] \nonumber \\
    \leq & \mathbb{E} \Bigg[ \Bigg\| \sum_{i\in \mathcal{S}_t} \frac{m_i}{M p_i^t } \bm{g}^{t}_{i} \Bigg\|_2^2 \Bigg] \nonumber \\
    \overset{(\text{b})}{\leq} & |\mathcal{S}^t| \sum_{i\in \mathcal{S}^t} \mathbb{E} \Bigg[ \Bigg\| \frac{m_i}{M p_i^t } \bm{g}^{t}_{i} \Bigg\|_2^2 \Bigg] \nonumber \\
    \overset{(\text{c})}{\leq} & |\mathcal{S}^t| \sum_{i\in \mathcal{S}^t} \left( \frac{m_i}{M p_i^t} \right)^2 G^2,
\end{align}
where (a) follows from the unbiasedness in Lemma \ref{unbiased}, i.e., $\mathbb{E} [\frac{m_i}{M p_i^t } \bm{g}^{t}_{i}] = \sum_{j\in \mathcal{N}} \frac{m_j}{M} \bm{g}^{t}_{j}$, (b) is from the inequality $\|\sum_{i\in \mathcal{S}^t} \bm{a} \|^2 \leq |\mathcal{S}^t| \sum_{i\in \mathcal{S}^t}\| \bm{a} \|^2, \forall \bm{a}$, and (c) follows from Assumption \ref{sgd2}.

Therefore, we conclude that there must exist a positive constant $C>0$ such that
\begin{equation}
    (1+\alpha) \mathbb{E}[e_{\text{com}}^{t}] + \left( 1+\frac{1}{\alpha} \right) \mathbb{E}[e_{\text{var}}^{t}] \leq C.
\label{eq:C}
\end{equation}

So far, we can show the convergence with the selected learning rates.
Specifically, the RHS of \eqref{eq:non-convex} can be expressed as follows:
\begin{align}
    & \lim_{T\rightarrow \infty} \frac{2}{\gamma_T}  \left( \mathbb{E}[f(\bm{w}^{0})] - f(\bm{w}^{*}) \right) + \frac{L}{\gamma_T}  \sum_{t=0}^{T-1} (\eta^t)^2 \left( 1+\frac{1}{\alpha} \right) \sigma^2 \nonumber \\
    & + \frac{L}{\gamma_T}  \sum_{t=0}^{T-1} (\eta^t)^2 \left[ (1+\alpha) \mathbb{E}[e_{\text{com}}^{t}] + \left( 1+\frac{1}{\alpha} \right) \mathbb{E}[e_{\text{var}}^{t}] \right] \nonumber \\
    \leq & \lim_{T\rightarrow \infty} \frac{2}{\gamma_T}  \left( \mathbb{E}[f(\bm{w}^{0})] - f(\bm{w}^{*}) \right) \nonumber \\
    & + \frac{L}{\gamma_T}  \sum_{t=0}^{T-1} (\eta^t)^2 \left( 1+\frac{1}{\alpha} \right) \sigma^2 + \frac{L}{\gamma_T}  \sum_{t=0}^{T-1} (\eta^t)^2 C \nonumber \\
    \leq & \lim_{T\rightarrow \infty} \frac{2}{\gamma_T}  \left( \mathbb{E}[f(\bm{w}^{0})] - f(\bm{w}^{*}) \right) \nonumber \\
    & + \frac{L}{\gamma_T}  \sum_{t=0}^{T-1} (\eta^t)^2 \left[ \left( 1+\frac{1}{\alpha} \right) \sigma^2 + C \right] \nonumber \\
    \overset{\text{(d)}}{=} & 0,
\end{align}
where (d) follows from that $\lim_{T\rightarrow \infty} \gamma_T = \infty$ and $\lim_{T\rightarrow \infty} \sum_{t=0}^{T-1} (\eta^t)^2 < \infty$.
\qed

\section{Proof of Theorem \ref{thm:convex}}\label{appendix-convex}
To prove Theorem \ref{thm:convex}, we first introduce the following virtual global model update obtained by aggregating local gradients from all the devices:
\begin{equation}
    \bm{v}^{t+1} = \bm{w}^t - \eta^t \sum_{i\in \mathcal{N}} \frac{m_i}{M} \bm{g}^{t}_{i}.
\end{equation} 
From Lemma \ref{unbiased}, we have $\mathbb{E} [\bm{w}^{t+1} - \bm{v}^{t+1}] = \mathbb{E} [\bm{w}^{t+1}] - \bm{v}^{t+1} = 0$.
Next, we can provide an upper bound for the distance between the global model $\bm{w}^{t+1}$ and the optimal model $\bm{w}^*$ via the virtual sequence $\{\bm{v}^{t}\}$ as follows:
\begin{align}
    & \mathbb{E} \left[ \left\| \bm{w}^{t+1} - \bm{w}^* \right\|_2^2 \right] \nonumber \\
    =& \mathbb{E} \left[ \left\| \bm{w}^{t+1} - \bm{v}^{t+1} + \bm{v}^{t+1} - \bm{w}^* \right\|_2^2 \right] \nonumber \\
    \overset{(\text{a})}{=} & \mathbb{E} \left[ \left\| \bm{w}^{t+1} - \bm{v}^{t+1} \right\|_2^2 \right] + \mathbb{E} \left[ \left\|  \bm{v}^{t+1} - \bm{w}^* \right\|_2^2 \right] \nonumber \\
    & + 2 \mathbb{E} \left[ \left\langle \bm{w}^{t+1} - \bm{v}^{t+1}, \bm{v}^{t+1} - \bm{w}^* \right\rangle \right] \nonumber \\
    \overset{(\text{b})}{=} & \mathbb{E} \left[ \left\| \bm{w}^{t+1} - \bm{v}^{t+1} \right\|_2^2 \right] + \mathbb{E} \Bigg[ \Bigg\|  \bm{w}^{t} - \eta^t \sum_{i\in \mathcal{N}} \frac{m_i}{M} \bm{g}^{t}_{i} - \bm{w}^* \Bigg\|_2^2 \Bigg] \nonumber \\
    \overset{(\text{c})}{=} & \mathbb{E} \Bigg[ \Bigg\| \eta^t \hat{\bm{y}}^t - \eta^t \sum_{i\in \mathcal{N}} \frac{m_i}{M} \bm{g}^{t}_{i} \Bigg\|_2^2 \Bigg]  \nonumber \\
    & + \mathbb{E} \left[ \left\|  \bm{w}^{t} - \bm{w}^* \right\|_2^2 \right]
    + (\eta^t)^2 \mathbb{E} \Bigg[ \Bigg\| \sum_{i\in \mathcal{N}} \frac{m_i}{M} \bm{g}^{t}_{i} \Bigg\|_2^2 \Bigg] \nonumber \\
    & - 2 \eta^t \mathbb{E} \Bigg[ \Bigg\langle \bm{w}^{t} - \bm{w}^*, \sum_{i\in \mathcal{N}} \frac{m_i}{M} \bm{g}^{t}_{i} \Bigg\rangle \Bigg] \nonumber \\
    = & (\eta^t)^2 \mathbb{E} \Bigg[ \Bigg\| \hat{\bm{y}}^t - \sum_{i\in \mathcal{S}^t} \frac{m_i}{M p_i^t} \bm{g}^{t}_{i} + \sum_{i\in \mathcal{S}^t} \frac{m_i}{M p_i^t} \bm{g}^{t}_{i} - \sum_{i\in \mathcal{N}} \frac{m_i}{M} \bm{g}^{t}_{i} \Bigg\|_2^2 \Bigg] \nonumber \\
    & + \mathbb{E} \left[ \left\|  \bm{w}^{t} - \bm{w}^* \right\|_2^2 \right]
    + (\eta^t)^2 \mathbb{E} \left[ \left\| \sum_{i\in \mathcal{N}} \frac{m_i}{M} \bm{g}^{t}_{i} \right\|_2^2 \right]  \nonumber \\
    & - 2 \eta^t \mathbb{E} \Bigg[ \Bigg\langle \bm{w}^{t} - \bm{w}^*, \sum_{i\in \mathcal{N}} \frac{m_i}{M} \bm{g}^{t}_{i} \Bigg\rangle \Bigg] \nonumber \\
    \overset{(\text{d})}{\leq} & (\eta^t)^2 (1+\alpha) \underbrace{\mathbb{E} \Bigg[ \Bigg\| \sum_{i\in \mathcal{S}^t} \frac{m_i}{M p_i^t} \bm{g}^{t}_{i} - \hat{\bm{y}}^t \Bigg\|_2^2 \Bigg]}_{\mathbb{E}[e_{\text{com}}^{t}]} \nonumber \\ 
    & + (\eta^t)^2  \left( 1+\frac{1}{\alpha} \right) \underbrace{\mathbb{E} \Bigg[\Bigg\| \sum_{i\in \mathcal{S}^t} \frac{m_i}{M p_i^t} \bm{g}^{t}_{i} - \sum_{i\in \mathcal{N}} \frac{m_i}{M} \bm{g}^{t}_{i} \Bigg\|_2^2 \Bigg]}_{\mathbb{E}[e_{\text{var}}^{t}]} \nonumber \\
    & + \mathbb{E} \left[ \left\|  \bm{w}^{t} - \bm{w}^* \right\|_2^2 \right]
    + (\eta^t)^2 \mathbb{E} \left[ \left\| \sum_{i\in \mathcal{N}} \frac{m_i}{M} \bm{g}^{t}_{i} \right\|_2^2 \right] \nonumber \\
    & - 2 \eta^t \mathbb{E} \left[ \left\langle \bm{w}^{t} - \bm{w}^*, \sum_{i\in \mathcal{N}} \frac{m_i}{M} \bm{g}^{t}_{i} \right\rangle \right],
    \label{eq:help-42}
\end{align}
where (a) and (c) follow from the equality $\|\bm{x}+\bm{y}\|_2^2 = \|\bm{x}\|_2^2+ \|\bm{y}^2\|_2 + 2\langle \bm{x}, \bm{y} \rangle$, (b) is because $\mathbb{E} [\bm{w}^{t+1}] = \bm{v}^{t+1}$, and (d) follows from the inequality $\left\|\bm{x} + \bm{y} \right\|_2^2 \leq (1+\alpha) \left\|\bm{x} \right\|_2^2 + \left( 1+\frac{1}{\alpha} \right) \left\|\bm{y} \right\|_2^2$, $\forall \bm{x}, \bm{y}$, and $\alpha>0$. 

By the convexity of the local loss function in Assumption \ref{convex}, we have:
\begin{equation}
    \mathbb{E} \left\langle \bm{w}^{t} - \bm{w}^*, \nabla f(\bm{w}^{t}) \right\rangle \geq \mathbb{E}[f(\bm{w}^{t})] - \mathbb{E}[f(\bm{w}^{*})].
    \label{eq:help-43}
\end{equation}
Meanwhile, by Assumption \ref{sgd2}, we have:
\begin{equation}
    \mathbb{E} \Big[ \Big\| \sum_{i\in \mathcal{N}} \frac{m_i}{M} \bm{g}^{t}_{i} \Big\|_2^2 \Big] \leq \sum_{i\in \mathcal{N}} \frac{m_i}{M} G^2 = G^2.
    \label{eq:help-44}
\end{equation}
Combining \eqref{eq:help-42}-\eqref{eq:help-44}, we have:
\begin{align}
    & \mathbb{E} \left[ \left\| \bm{w}^{t+1} - \bm{w}^* \right\|_2^2 \right] \nonumber \\
    \leq & (\eta^t)^2 \left[ (1+\alpha) \mathbb{E}[e_{\text{com}}^{t}] + \left( 1+\frac{1}{\alpha} \right) \mathbb{E}[e_{\text{var}}^{t}] \right]
    + \mathbb{E} \left[ \left\|  \bm{w}^{t} - \bm{w}^* \right\|_2^2 \right] \nonumber \\
    & + (\eta^t)^2 \mathbb{E} \left[ \left\| \sum_{i\in \mathcal{N}} \frac{m_i}{M} \bm{g}^{t}_{i} \right\|_2^2 \right] \nonumber \\
    & - 2 \eta^t \left( \mathbb{E}[f(\bm{w}^{t})] - \mathbb{E}[f(\bm{w}^{*})] \right) \nonumber \\
    \leq & \mathbb{E} \left[ \left\| \bm{w}^{t} - \bm{w}^* \right\|_2^2 \right] + (\eta^t)^2 \left[ (1+\alpha) \mathbb{E}[e_{\text{com}}^{t}] + \left( 1+\frac{1}{\alpha} \right) \mathbb{E}[e_{\text{var}}^{t}] \right]  \nonumber \\
    & + (\eta^t)^2 G^2
    - 2\eta^t \left( \mathbb{E}[f(\bm{w}^{t})] - f(\bm{w}^{*}) \right).
\label{eq:help-1}
\end{align}

By rearranging the terms in \eqref{eq:help-1}, summing both sides up over $t=0,1,\dots, T-1$, and dividing them by $2\gamma_T = 2 \sum_{t=0}^{T-1} \eta^t $, we have:
\begin{align}
    & \frac{1}{\gamma_T} \sum_{t=0}^{T-1} \eta^t \left( \mathbb{E}[f(\bm{w}^{t})] - f(\bm{w}^{*}) \right)  \nonumber \\
    \leq & \frac{1}{2\gamma_T} \sum_{t=0}^{T-1} \mathbb{E} \left[ \left\| \bm{w}^{t} - \bm{w}^* \right\|_2^2 \right] - \frac{1}{2\gamma_T} \sum_{t=0}^{T-1} \mathbb{E} \left[ \left\| \bm{w}^{t+1} - \bm{w}^* \right\|_2^2 \right] \nonumber \\
    & + \frac{1}{2\gamma_T} \sum_{t=0}^{T-1} (\eta^t)^2 G^2 \nonumber \\
    & + \frac{1}{2\gamma_T} \sum_{t=0}^{T-1} (\eta^t)^2 \left[ (1+\alpha) \mathbb{E}[e_{\text{com}}^{t}] + \left( 1+\frac{1}{\alpha} \right) \mathbb{E}[e_{\text{var}}^{t}] \right] \nonumber \\
    \leq & \frac{1}{2\gamma_T} \! \mathbb{E} \left[ \left\| \bm{w}^{0} - \bm{w}^* \right\|_2^2 \right]
    \!+\! \frac{1}{2\gamma_T} \! \sum_{t=0}^{T-1} (\eta^t)^2  G^2 \nonumber \\
    & + \frac{1}{2\gamma_T} \! \sum_{t=0}^{T-1} (\eta^t)^2 \left[ (1+\alpha) \mathbb{E}[e_{\text{com}}^{t}] \!+\! \left( 1+\frac{1}{\alpha} \right) \mathbb{E}[e_{\text{var}}^{t}] \right]. \label{eq:help-46}
\end{align}

Moreover, since $ \min_{t\in[T]} a_t \leq \sum_{t=0}^{T-1} \frac{\eta^t}{\sum_{t=0}^{T-1} \eta^t} a_t$ holds for any $a_t$, we have:
\begin{align}
    \mathbb{E}[f(\tilde{\bm{w}}^T)] - f(\bm{w}^{*}) 
    =& \min_{t\in[T]} \mathbb{E} [f(\bm{w}^t)] - f(\bm{w}^{*}) \nonumber \\
    \leq & \left( \sum_{t=0}^{T-1} \frac{\eta^t}{\gamma_T} \mathbb{E}[f(\bm{w}^{t})] \right)- f(\bm{w}^{*}) \nonumber \\
    = & \frac{1}{\gamma_T} \sum_{t=0}^{T-1} \eta^t \left( \mathbb{E}[f(\bm{w}^{t})] - f(\bm{w}^{*}) \right).
\label{eq:help-47}
\end{align}
Substituting \eqref{eq:help-46} into \eqref{eq:help-47} gives the results in \eqref{eq:convex}.
\qed

\section{Proof of Corollary \ref{corollary1}} \label{proof-corollary1}

From \eqref{eq:C}, we can express the RHS of \eqref{eq:convex} as follows:

\begin{align}
    & \lim_{T\rightarrow \infty} \frac{1}{2\gamma_T} \sum_{t=0}^{T-1} \left( 1 - \mu\eta^0 \right) \mathbb{E} \left[ \left\| \bm{w}^{0} - \bm{w}^* \right\|_2^2 \right] + \frac{1}{2\gamma_T} \sum_{t=0}^{T-1} (\eta^t)^2 G^2 \nonumber \\
    & + \frac{1}{2\gamma_T} \sum_{t=0}^{T-1} (\eta^t)^2 \left[ (1+\alpha) \mathbb{E}[e_{\text{com}}^{t}] + \left( 1+\frac{1}{\alpha} \right) \mathbb{E}[e_{\text{var}}^{t}] \right] \nonumber \\
    \leq & \lim_{T\rightarrow \infty} \frac{1}{2\sum_{t=0}^{T-1} \eta^t} \times \nonumber \\
    & \underbrace{\left[ \sum_{t=0}^{T-1} \left( 1 - \mu\eta^0 \right) \mathbb{E} \left[ \left\| \bm{w}^{0} - \bm{w}^* \right\|_2^2 \right] \!+\! \sum_{t=0}^{T-1} (\eta^t)^2 C \!+\! \sum_{t=0}^{T-1} (\eta^t)^2 G^2 \right]}_{C_1} \nonumber \\
    \overset{(\text{a})}{=} & 0,
\end{align}
where (a) is because $C_1$ is a constant irrespective of $t$ and $\lim_{T\rightarrow \infty} \gamma_T = 0$.
\qed

\bibliographystyle{IEEEtran}
\bibliography{bibliofile}

\end{document}